\documentclass[11pt]{article}

\usepackage{etoolbox}
\providetoggle{submission}
\providetoggle{arxiv}

\settoggle{submission}{false}
\settoggle{arxiv}{true}

\usepackage[utf8]{inputenc} 
\usepackage[T1]{fontenc} 
\usepackage{environ,comment}
\usepackage{amsfonts,bbm,bm,courier,nicefrac,microtype}
\usepackage{amsmath,amsthm}
\usepackage{eqnarray}
\usepackage{enumitem}
\usepackage{array,booktabs,tabularx,multirow,colortbl,hhline}
\usepackage{graphicx,xcolor,placeins}
\usepackage{todonotes}

\usepackage{hyperref}
\hypersetup{colorlinks=true, urlcolor=blue, linkcolor=blue, citecolor=blue, linkbordercolor=blue, pdfborderstyle={/S/U/W 1}}

\usepackage{natbib}
\usepackage[capitalise]{cleveref}

\newlength\tablewidth

\iftoggle{arxiv}{
    \usepackage[preprint]{ext/jmlr2e}
    \usepackage[font=small,labelfont=bf,width=\textwidth]{caption}    
    \setcitestyle{numbers,comma,square,sort}
}

\setitemize{leftmargin=0.5em,topsep=0.1em,parsep=0.5pt,label=\raisebox{0.25ex}{\tiny$\bullet$}}
\setlist[enumerate]{leftmargin=*, label= {\arabic*.}, itemsep=0.5em}

\newcommand{\textds}[1]{{\footnotesize{\texttt{#1}}}}
\newcommand{\textfn}[1]{\texttt{\small{#1}}}
\newcommand{\textmn}[1]{\textsf{\small{#1}}}
\newcommand{\textgp}[1]{\texttt{\footnotesize{#1}}}

\newcolumntype{H}{>{\setbox0=\hbox\bgroup}c<{\egroup}@{}}
\newcommand{\textheader}[1]{{\bfseries{#1}}}

\newcommand{\cell}[2]{\setlength{\tabcolsep}{0pt}\begin{tabular}{#1}#2 \end{tabular}}

\definecolor{fgood}{HTML}{009E73}
\definecolor{fbad}{HTML}{D55E00}
\definecolor{funk}{HTML}{F0E442}

\newcommand{\ftblmidrules}{\hline}

\newcommand{\ftblheader}[3]{\multicolumn{#1}{#2}{\cell{#2}{#3}}}

\newcommand\clearrow{\global\let\rowmac\relax}
\definecolor{best}{HTML}{BAFFCD}
\definecolor{bad}{HTML}{FFC8BA}
\newcommand{\good}[1]{\cellcolor{fgood!60}#1} 
\newcommand{\violation}[1]{\cellcolor{fbad!60}#1}
\newcommand{\closecall}[1]{\cellcolor{funk!60}#1} 

\DeclareMathOperator*{\argmin}{argmin}

\newcommand{\sign}[1]{\textnormal{sign}\left(#1\right)}
\newcommand{\prob}[1]{\textrm{Pr}(#1)}
\newcommand{\R}{\mathbb{R}}
\newcommand{\E}[0]{\mathbb{E}}

\newcommand{\xb}{\bm{x}}
\newcommand{\X}{\mathcal{X}}
\newcommand{\Y}{\mathcal{Y}}

\newcommand{\txplus}[0]{+}
\newcommand{\txminus}[0]{-}

\newcommand{\nplus}[1]{n^{\txplus{}}_{#1}}
\newcommand{\nminus}[1]{n^{\txminus{}}_{#1}}
\newcommand{\n}[1]{n_{#1}}

\newcommand{\gb}{\bm{g}}
\newcommand{\G}{\mathcal{G}}

\newcommand{\rb}{\bm{r}}

\newcommand{\Hset}[0]{\mathcal{H}}

\newcommand{\truerisk}[2]{{R_{#1}(#2)}}
\newcommand{\emprisk}[2]{{\hat{R}_{#1}(#2)}}

\newcommand{\empgap}[3]{\hat{\Delta}_{#1}({#2},{#3})}
\newcommand{\truegap}[3]{\Delta_{#1}({#2}, {#3})}
\newcommand{\empgapmin}[1]{\hat{\epsilon}_{#1}}

\newcommand{\truegain}[3]{{\Delta_{#1}}({#2},{#3})}
\newcommand{\empgain}[3]{{\hat{\Delta}_{#1}}({#2},{#3})}

\newcommand{\clf}[0]{h_0}
\newcommand{\plf}[1]{h_{#1}}
\newcommand*\samethanks[1][\value{footnote}]{\footnotemark[#1]}

\iftoggle{arxiv}{
\jmlrheading{1}{2022}{1-48}{4/00}{10/00}{}{Vinith M. Suriyakumar, Marzyeh Ghassemi, and Berk Ustun}
\ShortHeadings{When Personalization Harms Performance}{Suriyakumar, Ghassemi \& Ustun}
\title{When Personalization Harms Performance:\\Reconsidering the Use of Group Attributes in Prediction}
\author{%
    \name Vinith M. Suriyakumar \email vinithms@mit.edu\\
     \addr Massachusetts Institute of Technology 
    \AND 
    \name Marzyeh Ghassemi\thanks{Equal Supervision}\email mghassem@mit.edu\\
    \addr Massachusetts Institute of Technology 
    \AND
    \name Berk Ustun\samethanks \email berk@ucsd.edu\\
    \addr University of California, San Diego
}%
}

\begin{document}

\maketitle

\begin{abstract}
Machine learning models are often personalized with categorical attributes that define groups. In this work, we show that personalization with \emph{group attributes} can inadvertently reduce performance at a \emph{group level} -- i.e., groups may receive unnecessarily inaccurate predictions by sharing personal characteristics.
We present formal conditions to ensure the \emph{fair use} of group attributes in a prediction task that can be checked by training one additional model. 
We characterize how fair use conditions can be violated due to standard practices in model development, and study their prevalence in clinical prediction tasks.
Our results show that personalization often fails to produce a tailored performance gain for every group that reports personal data, and underscore the need to check fair use when personalizing models with characteristics that are protected, sensitive, self-reported, and costly to acquire.
\end{abstract}

\iftoggle{arxiv}{\vspace{0.5em}
}

\section{Introduction}
\label{Sec::Introduction}

Machine learning models are often used to assign predictions to people -- be it to predict if a patient has a rare disease, the risk that a consumer will default on a loan, or the likelihood that a student will matriculate. 

Models in such applications are often \emph{personalized} to target heterogeneous subgroups. In the most common approach, models are trained with \emph{group attributes} -- i.e., categorical attributes that define groups. In consumer finance, credit scores may include group attributes that are \emph{protected} such as \textds{age\_group}~\citep[][]{ECOA20}. In medicine, clinical prediction models may include group attributes that are \emph{sensitive} (e.g., \textgp{AIDS} as in the \href{https://www.mdcalc.com/calc/4044/simplified-acute-physiology-score-saps-ii}{SAPS II Score}), \emph{self-reported} (e.g. \textgp{sexual\_practices} as in the \href{https://www.mdcalc.com/calc/10046/denver-hiv-risk-score}{Denver HIV Risk Score}), or \emph{costly to acquire} (e.g., \href{https://www.osmind.org/knowledge-article/measurement-based-care-mental-health-patient-rating-scales}{Brief Psychiatry Rating Scale}). 

The widespread use of group attributes in modern prediction models reflects a belief that personalization can only improve performance. In effect, practitioners who develop clinical prediction models include protected attributes like \textgp{race} because they believe it can only improve performance~\citep[][]{kent2020predictive}. Likewise, individuals report sensitive attributes like \textds{sexual\_practices} to a self-reported screening tool because they expect to receive more accurate predictions. 

In this work, we formalize these expectations through a principle that we call \emph{fair use} –- i.e., that every person who reports personal characteristics should expect a tailored gain in performance in return. Given a model that is personalized with group attributes, we then test that it satisfies these minimal expectations. First, by testing that every group expects more accurate predictions from a personalized compared to a \emph{generic model} trained without their group attributes. Next, by testing that the gains are \emph{tailored}, meaning that every group prefers their personalized predictions to predictions personalized for any other group.

The vast majority of machine learning models are not designed to ensure fair use (see \cref{Fig::FairUseViolationPractice}). This result stems from the fact that standard approaches to empirical risk minimization use group attributes to improve performance at a \emph{population level}. As we will show, the resulting models may assign unnecessarily inaccurate predictions at the group level due to routine decisions such as model specification and model selection (see \cref{Fig::FailureModeMisspecification}). 

In practice, however, these fair use violations may inflict harm. In clinical applications, for example, inaccurate predictions lead to worse decisions and health outcomes~\citep{vyas2020hidden}. More broadly, these effects are silent and avoidable. Silent as fair use violations would only draw attention if we were to evaluate the \emph{gains of personalization} for \emph{intersectional} groups. Avoidable because a fair use violation implies that a group could receive better predictions from a generic model or a personalized model for another group. Thus, one could resolve a fair use violation by assigning predictions from this better-performing model.

Our goal in this work is to expose this effect and lay the foundations to address it. Our main contributions include:
\begin{enumerate}[itemsep=0.005em,topsep=-0.01em]

\item We present formal conditions to ensure the fair use of group attributes in prediction task. Our conditions reflect collective preference guarantees that are necessary for truthful self-reporting, and that can be tested by training one additional model.

\item  We characterize how empirical risk minimization with group attributes can violate fair use. Our analysis includes counterexamples and sufficient conditions that illustrate failure modes in model development and inform interventions to mitigate their effects.

\item We conduct a comprehensive empirical study on fair use violations in clinical prediction tasks, showing their prevalence across major model classes, personalization techniques, and prediction tasks.

\item We present a case study on personalization for a model trained to predict mortality for patients with acute kidney injury. Our study shows how a fair use audit can safeguard against incorrect ``race correction" in clinical prediction models, and presents targeted interventions that reduce harm.

\end{enumerate}

\newcommand{\titlecolumn}[3]{\multicolumn{#1}{#2}{#3}}

\begin{figure}[!t]
    \centering
    \setlength{\tablewidth}{\iftoggle{arxiv}{0.4\textwidth}{\tablewidth}}
    
    \resizebox{\tablewidth}{!}{
    \begin{tabular}{@{}rr*{2}{r}r}
    \titlecolumn{1}{c}{Group} & 
    \titlecolumn{1}{c}{Size} & 
    \titlecolumn{2}{c}{Training Error} &  \titlecolumn{1}{c}{Gain} \\ 
    \cmidrule(lr){1-1}\cmidrule(lr){2-2}\cmidrule(lr){3-4} \cmidrule(lr){5-5} 
    \multicolumn{1}{c}{$\gb{}$} & $\n{\gb}$ &  $\truerisk{}{\clf{}}$ &  $\truerisk{\gb}{\plf{}}$ & $\truegain{\gb}{\plf{}}{\clf{}}$  \\
    \cmidrule(lr){1-1} \cmidrule(lr){2-2}  \cmidrule(lr){3-4}  \cmidrule(lr){5-5} 
    \textgp{female, <30} &   48 &               38.1\% &                    26.8\% &        \good{11.3\%} \\ 
    \textgp{male, <30} &   49 &              23.9\% &                    26.7\% &        \violation{-2.8\%} \\ 
    \textgp{female, 30 to 60} &   304 &               30.3\% &                    29.1\% &        \good{1.2\%} \\ 
    \textgp{male, 30 to 60} &   447 &               15.4\% &                    15.2\% &        \good{0.2\%} \\ 
    \textgp{female, 60+} &   123 &               19.3\% &                    21.9\% &        \violation{-2.6\%} \\        
    \textgp{male, 60+} &   181 &               11.0\% &                    8.2\% &       \good{2.8\%} \\            
    \cmidrule(lr){1-1} \cmidrule(lr){2-2}  \cmidrule(lr){3-4}  \cmidrule(lr){5-5} 
    \textbf{Total} &  \color{black}{1,152} &               \color{black}{20.4\%} &                    \color{black}{19.4\%} &  \cellcolor{fgood!60} 1.0\% 
    \end{tabular}
    }
    \caption{Personalization can reduce performance at the group level. We train a personalized model $\plf{\gb}$ and generic model $\plf{0}$ with logistic regression, personalizing $\plf{\gb}$ with a one-hot encoding of $\textgp{sex}$ and $\textgp{age\_group}$ to screen for obstructive sleep apnea (see the \textds{apnea} dataset in \cref{Sec::Experiments}). As shown, personalization reduces training error at a population level from 20.4\% to 19.4\% yet \emph{increases} error for two groups: (\textgp{female, 60+}) and (\textgp{male, <30}). These effects are also present on test data.}
    \label{Fig::FairUseViolationPractice}
\end{figure}

\paragraph{Related Work}

Personalization encompasses a broad range of techniques that use personal data. Here, we use it to describe techniques that target \emph{groups} rather than \emph{individuals} -- i.e., ``categorization" rather than ``individualization" as per the taxonomy of \citet{fan2006personalization}. Modern approaches to personalization with group attributes use them to improve population-level performance by, e.g., automatically including higher-order interaction effects~\citep{bien2013lasso,lim2015learning,vaughan2020efficient} or recursively partitioning data~\citep{elmachtoub2018value,biggs2020model,bertsimas2019optimal,bertsimas2020predictive}. In practice, few works measure the gains of personalization, and those that do measure the gains at a population level rather than a group level~\citep[see e.g.,][]{jaques2016multitask,taylor2017personalized}. 

We introduce conditions for models that use group attributes to assign more accurate predictions. Much work in algorithmic fairness discusses the need for models to account for group membership~\citep[][]{zafar2017parity,dwork2017decoupled,corbett2018measure,kleinberg2018algorithmic,lipton2018disparate,wang2019repairing}, observing that it is otherwise impossible for a model to perform equally well for all groups~\citep[][]{hardt2016equality,zafar2017fairness,zafar2015fairness,feldman2015certifying,agarwal2018reductions,narasimhan2018learning,celis2019classification}. These results highlight the need to account for group attributes in personalization. Nevertheless, methods to equalize performance are ill-suited for personalization because they can equalize performance by assigning less accurate predictions to groups for whom the model performs well, rather than by assigning more accurate predictions to groups for whom the model performs poorly~\citep[][]{lipton2018disparate,hu2019fair,pfohl2019creating,martinez2019fairness,martinez2020minimax}.

We build on the work of~\citet{ustun2019fairness}, who propose the preference guarantees of rationality and envy-freeness~\citep[see also][]{zafar2017parity}. Their work develops a recursive decoupling algorithm that uses preference guarantees to guide decoupling~\citep[c.f.,][]{dwork2017decoupled,alabi2018unleashing}. In contrast, we study these guarantees as standalone conditions to ensure personalization without harm. Our work complements an emerging stream on fair use in prediction models~\citep[see e.g.,][]{monteiro2022epistemic,james2023participatory}. More broadly, it highlights a practical application for preference-based notions of fairness~\citep{zafar2017parity,ustun2019fairness,kim2019preference,viviano2020fair,do2021online}, and represents a new use case to evaluate model performance across intersectional groups~\citep[c.f.,][]{kearns2017preventing,hebert2018multicalibration,globus2022algorithmic,wang2022towards}.

\begin{figure}[!t]
\centering
\setlength{\tablewidth}{0.5\textwidth}
\resizebox{\tablewidth}{!}{
\begin{tabular}{rr>{\;}r>{\;}r>{\;}rrrr}
    \titlecolumn{1}{c}{Group} &
    \titlecolumn{2}{c}{Data} &
    \titlecolumn{2}{c}{Personalized}  &
    \titlecolumn{2}{c}{Generic}&  
    \titlecolumn{1}{c}{Gain}\\
    \cmidrule(lr){1-1}\cmidrule(lr){2-3}\cmidrule(lr){4-5}\cmidrule(lr){6-7}\cmidrule(lr){8-8}
     \multicolumn{1}{c}{$\gb{}$} & 
     $\nplus{\gb}$ & $\nminus{\gb}$  & 
     $\plf{}$ & $R_{\gb}(\plf{})$ &
     $h_0$ & $R_{\gb}(\clf{})$ & 
     $\truegain{\gb}{\plf{}}{\clf{}}$\\
     \cmidrule(lr){1-1}\cmidrule(lr){2-3}\cmidrule(lr){4-5}\cmidrule(lr){6-7}\cmidrule(lr){8-8}
     $\textgp{female, young}$ & 
     $0$ & $24$ &
     $+$ & $24$ & 
     $-$ & $0$ &
     \cellcolor{fbad!60} $-24$ \\       
     $\textgp{male, young}$ & 
     $25$ & $0$ &
     $+$ & $0$ & 
     $-$ &  $25$ &
     \cellcolor{fgood!60} $25$ \\
     $\textgp{female, old}$ & 
     $25$ & $0$ & 
     $+$ & $0$ & 
     $-$ & $25$ &
     \cellcolor{fgood!60} $25$ \\
    $\textgp{male, old}$ 
     & $0$ & $27$ &
     $-$ & $0$ & 
     $-$ & $0$ &
     \cellcolor{funk!60} $0$ \\
     \cmidrule(lr){1-1}\cmidrule(lr){2-3}\cmidrule(lr){4-5}\cmidrule(lr){6-7}\cmidrule(lr){8-8}
     \multicolumn{1}{r}{\bfseries Total} & 
     50 & 51 & & 24  & & 50 & \cellcolor{fgood!60} 26 
\end{tabular}
}
\caption{Stylized classification task where the best personalized model reduces performance for a group due to model misspecification. There are $\nplus{} = 50$ positive and $\nminus{} = 51$ negative examples. We train a personalized linear classifier with a one-hot encoding of $\gb \in \{\textgp{male}, \textgp{female}\} \times \{\textgp{old}, \textgp{young}\}$, and evaluate the gains to personalization with respect to a generic model $h_0$ without group attributes.
Personalization reduces overall error from 50 to 24. However, not all groups gain from personalization: $(\textgp{young, female})$ receives less accurate predictions and $(\textgp{old, male})$ receives no gain.} 
\label{Fig::FailureModeMisspecification}
\vspace{-0.2em}
\end{figure}

\section{Fair Use Conditions}
\label{Sec::ProblemStatement}

We present formal conditions for the fair use of group attributes in prediction tasks.

\paragraph{Preliminaries}

We start with a dataset $(\xb_i, y_i, \gb_i)_{i=1}^n$, where example $i$ consists of a feature vector $\xb_i=[x_{i,1}, \ldots, x_{i,d}] \in \R^{d}$, a label $y_i \in \Y$, and $k$ categorical \emph{group attributes} $\gb_i = [g_{i,1},\ldots,g_{i,k}] \in \G_1 \times \ldots \times \G_k = \G$. We refer to $\gb_i$ as the \emph{group membership} of person $i$. For example, a female over 60 would have $\gb_i = [\textgp{female}, \,\textgp{age}\geq\textgp{60}]$. We let $\n{\gb} := |\{i ~| \gb_i = \gb \}|$ denote the size of group $\gb$, and $m := |\G|$ denote the number of intersectional groups. 

We use the data to train a \emph{personalized} model with group attributes $\plf{}: \X \times \G \to \Y$; and a  \emph{generic} model that does not $\clf{}:\X \to \Y$. We train all models via ERM with a loss function $\ell:\Y\times\Y\to\R_+$, denoting the empirical and true risks as $\emprisk{}{h}$ and $\truerisk{}{h}$, respectively. We assume that the personalized and generic models represent the empirical risk minimizers on datasets with group attributes $(\xb_i, y_i, \gb_i)_{i=1}^n$ and without them $(\xb_i,y_i)_{i=1}^n$:
\begin{align*}
    \plf{} \in \argmin_{h \in \Hset} \emprisk{}{h}\qquad \clf{} \in \argmin_{h \in \Hset_0}\emprisk{}{h}
\end{align*}
Here, $\Hset$ and $\Hset_0$ denote the class of personalized models and generic models respectively.

We measure the \emph{gains of personalization} for a personalized model $\plf{}$ for each group. As part of this evaluation, we measure how the model will perform for group $\gb$ when they are assigned the predictions personalized for a different group -- i.e., the predictions that they could receive by ``misreporting" their group membership as $\gb'$. 
Given a personalized model $\plf{}$, we denote its \emph{empirical risk} and \emph{true risk} for group $\gb$ when they report $\gb'$ as:
\iftoggle{arxiv}{
\begin{align*}
\emprisk{\gb}{\plf{\gb'}} &:= \frac{1}{\n{\gb}} \sum_{i:\gb_i = \gb} \ell\left( h(\xb_i, \gb'),y_i \right) \qquad \truerisk{\gb}{\plf{\gb'}} := \E\left[\ell\left(h(\xb,\gb'),y \right) \mid \G=\gb\right].
\end{align*}
}{\begin{align*}
\emprisk{\gb}{\plf{\gb'}} &:= \frac{1}{\n{\gb}} \sum_{i:\gb_i = \gb} \ell\left( h(\xb_i, \gb'),y_i \right)\\
\truerisk{\gb}{\plf{\gb'}} &:= \E\left[\ell\left(h(\xb,\gb'),y \right) \mid \G=\gb\right].
\end{align*}
}
We use $\plf{\gb{}'} := h(\cdot, \gb{}')$ to denote a personalized model where group attributes are fixed to $\gb{}'$.

We assume that each group prefers models that assign more accurate predictions as measured in terms of true risk, and evaluate the preferences of group $\gb$ between $h$ and $h'$ using the \emph{gain} measure: $\truegap{\gb}{h}{h'} := \truerisk{\gb}{h'} - \truerisk{\gb}{h}$. This is a plausible assumption in settings where models are used to assign personalized predictions. It does not hold in settings where individuals may prefer models that ~\citep[see e.g., polar prediction tasks][]{paulus2020predictably}.

\paragraph{As Collective Preference Guarantees}

In Definition~\ref{Def::FairUse}, we characterize the fair use of a group attribute in terms of collective preference guarantees. 
\begin{definition}
\label{Def::FairUse}
A personalized model $\plf{}: \X \times \G \to \Y$ guarantees the fair use of group attributes $\G$ if it obeys: 
\begin{align}
\truegap{\gb}{\plf{\gb}}{h_0} &\geq 0 &\; \textrm{for all groups}\; \gb \in \G, \label{Def::Rationality} \\ %(rationality)
\truegap{\gb}{\plf{\gb}}{\plf{\gb'}} &\geq 0 &\;\textrm{for all groups}\; \gb, \gb' \in \G \label{Def::Envyfreeness} %(envy-freeness)
\end{align}
\end{definition}
These conditions are collective in that performance is measured over individuals in a group. 
Here, condition \eqref{Def::Rationality} ensures \emph{rationality} for group $\gb$ -- i.e., that a majority of group $\gb$ prefers a personalized model $\plf{\gb}$ to a generic model $\clf{}$. Condition \eqref{Def::Envyfreeness} ensures \emph{envy-freeness} for group $\gb$ -- i.e., that majority of group $\gb$ prefers their personalized predictions to the personalized predictions for any other group. These conditions reflect minimal expectations of groups from a personalized model.
% Without rationality, a majority in some group would prefer the generic model. Without envy-freeness, a majority in some group would prefer the personalized predictions assigned to another group. 

These conditions can be adapted to different supervised learning tasks by choosing a suitable risk metric. Since fair use conditions reflect the expected gains from personalization, a ``suitable" metric should represent an exact measure of model performance rather than a surrogate measure optimized for training. In classification tasks where we want accurate predictions, this would be the error rate. In tasks where we want reliable risk estimates, it would be the expected calibration error~\citep{naeini2015obtaining}.

\paragraph{As Prerequisites for Truthful Self-Reporting}

In copyright law, fair use conditions characterize when we can use copyrighted material without permission from copyright owners~\citep{yankwich1954fair,netanel2011making}. In this setting, fair use conditions characterize when we can use personal data without asking permission from the owners of that data. In particular, fair use conditions are necessary for ``truthful self-reporting"~\citep[see e.g.,][]{savage1971elicitation,jovanovic1982truthful,gneiting2007strictly}. 

\begin{proposition}
\label{Rem::TruthfulSelfReporting}
Consider a prediction model where each person reports their group membership to a personalized model $\plf{}: \X \times \G \to \Y$ in deployment. Denote the reported group membership of person $i$ as:
\begin{equationarray}{lclcl}
    \rb_i & = & \gb_i                   & \Leftrightarrow &~\textrm{$i$ reports truthfully}  \nonumber \\
    \rb_i & \in & \G\setminus \{\gb_i\}  & \Leftrightarrow & ~\textrm{$i$ misreports} \nonumber \\
    \rb_i & = & ?                       & \Leftrightarrow & ~\textrm{$i$ withholds}  \nonumber
\end{equationarray}
If a personalized model guarantees the fair use of $\G$ then each person would choose to report truthfully as this strategy would maximize their expected performance: $$\gb_i \in \argmin_{\rb_i \in \G \cup \{?\}} \E\left[\ell\left(h(\xb,\rb_i),y_i \right)\mid\G=\gb_i\right].$$
\end{proposition}
Truthful self-reporting incentives reflect basic principles regarding \emph{consent} in data privacy. In effect, a personalized model that violates fair use uses group membership in a way that is coercive. If groups were allowed to report personal information to a personalized model at prediction time, group who experience a fair use violation would not report group membership voluntarily or truthfully, choosing to withhold or misreport instead. If a model obeys fair use, individuals may still withhold group membership because the gain is insufficient. In light of this, fair use conditions should be viewed as minimal requirements to flag harm rather than a ``rubber stamp" for consent.

\paragraph{Use Cases} 

Fair use conditions should hold in prediction tasks where individuals are entitled to control or report their own data.
In such tasks, we should ensure fair use conditions for group attributes that encode:
\begin{itemize}[leftmargin=0pt,label={},itemsep=0.25em,topsep=0.1em]
    
    \item {\em Immutable Attributes}: Group attributes often encode characteristics like \textgp{sex}~\citep[see e.g.,][]{paulus2016field}. In this setting, fair use conditions ensure that individuals will not receive unnecessarily inaccurate predictions due to immutable characteristics.

    \item {\em Sensitive Information}: Models that use attributes like $\textgp{hiv\_status}$ should guarantee a tailored gain in performance for the sensitive group, $\textgp{hiv\_status}=\textgp{+}$. Otherwise, they require individuals to disclose information that may be harmful when leaked~\citep[see e.g.,][]{bansal2010impact}. 
    
    \item {\em Self-Reported Information}: Models are often personalized using information that individuals report directly -- see e.g., self-reported screening tests for mental illnesses~\citep[ ][]{kessler2005world,ustun2017world}. These models should obey fair use conditions to incentivize truthful self-reporting as per \cref{Rem::TruthfulSelfReporting}.

    \item {\em Costly Information}: Group attributes can encode characteristics that must be collected at test time -- e.g., an attribute like \textgp{pcr\_test} whose value requires a medical test. Models that ensure fair use with respect to \textgp{pcr\_test} guarantee that patients with a specific outcome will not receive a less accurate prediction after taking a test.
    
\end{itemize}

\paragraph{Testing for Fair Use}
\label{Sec::HypothesisTest}
We can evaluate fair use conditions by training a generic model in addition to a personalized model. Given a personalized model and its generic counterpart, we can check the conditions in \cref{Def::FairUse} on a sample by computing the relevant performance gains. This procedure will return point estimates that should be paired with a measure of uncertainty to guide model development. In some tasks, a significant fair use violation may warrant a new model. In others, we may wish to ensure a significant gain to use a group attribute in the first place. 

In practice, we check for a rationality violation using a one-sided hypothesis test of the form:
\begin{align*}
H_0: \truerisk{\gb}{\clf{}} - \truerisk{\gb}{\plf{\gb}} \leq 0\\
H_A: \truerisk{\gb}{\clf{}} - \truerisk{\gb}{\plf{\gb}} > 0
\end{align*}
Here, the null hypothesis $H_0$ assumes that group $\gb$ prefers $h_{\gb}$ to $h_0$. Thus, we would reject $H_0$ when there is enough evidence to support a rationality violation for $\gb$ on held-out data. %

We can test all conditions in \cref{Def::FairUse} by repeating this test for all $m$ groups to check rationality, and repeating analogous tests for all $m(m-1)$ pairs of groups to check envy-freeness. In general, one can test these hypotheses for any performance metric using a bootstrap hypothesis test~\citep{diciccio1996bootstrap}, and control the false discovery rate using a Bonferroni correction~\citep[][]{dunn1961multiple}. In practice, one should draw on more powerful tests when working with salient performance metrics~\citep[e.g., the McNemar test for accuracy][]{dietterich1998approximate}.

\section{Failure Modes and Guarantees}
\label{Sec::FailureModes}

Practitioners naturally presume that training a model with group attributes will provide a uniform performance gain for all groups. Here, we characterize how empirical risk minimization may fail to improve performance at a group level through counterexamples and sufficient conditions. We include additional examples and proofs in \cref{Appendix::Theory}.

\subsection{Failure Modes}

We characterize common practices that lead personalization to reduce performance at a group level. We present examples related to model misspecification and model selection as they motivate interventions for model development in \cref{Sec::Experiments}. We include examples related to generalization, distributional shifts, and training with a surrogate loss function in \cref{Appendix::FailureModes}.

\newcommand{\addplot}[1]{%
\cell{c}{%
%={left lower right upper}%
\includegraphics[trim=0.0in 0.0in 0.0in 0.0in,clip, width=0.3\linewidth]{#1}}%
}

\newcommand{\addplotgen}[1]{%
\cell{c}{%
%={left lower right upper}%
\includegraphics[trim=0.1in 0.3in 0.0in 0.0in,clip, width=0.255\linewidth]{#1}}%
}

\begin{figure*}[t]
\centering
\begin{tabular}{@{}ccc@{}}
\textsf{Bayes Optimal} & 
\textsf{Personalized} & 
\textsf{Generic} \\
\addplot{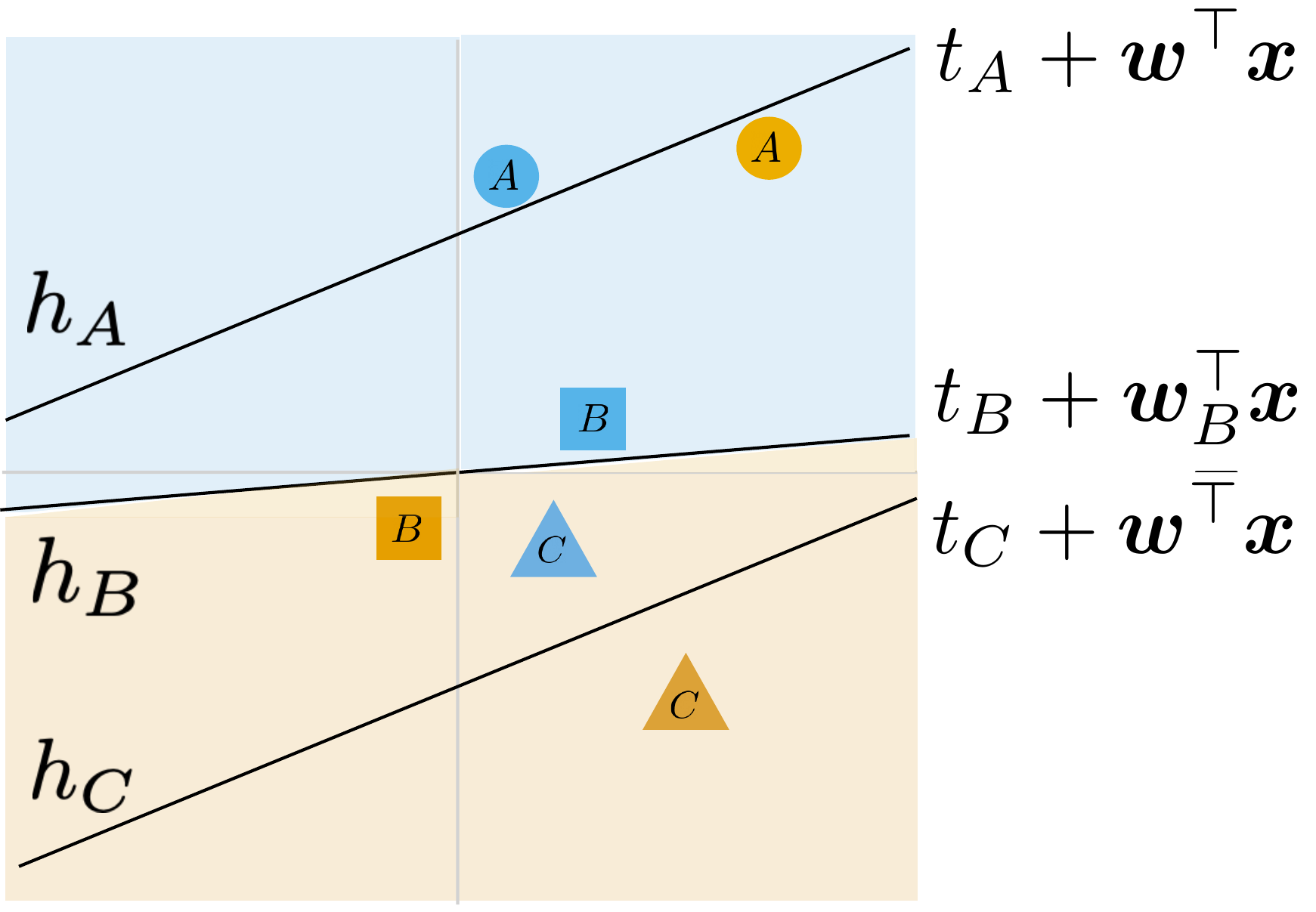} &
\addplot{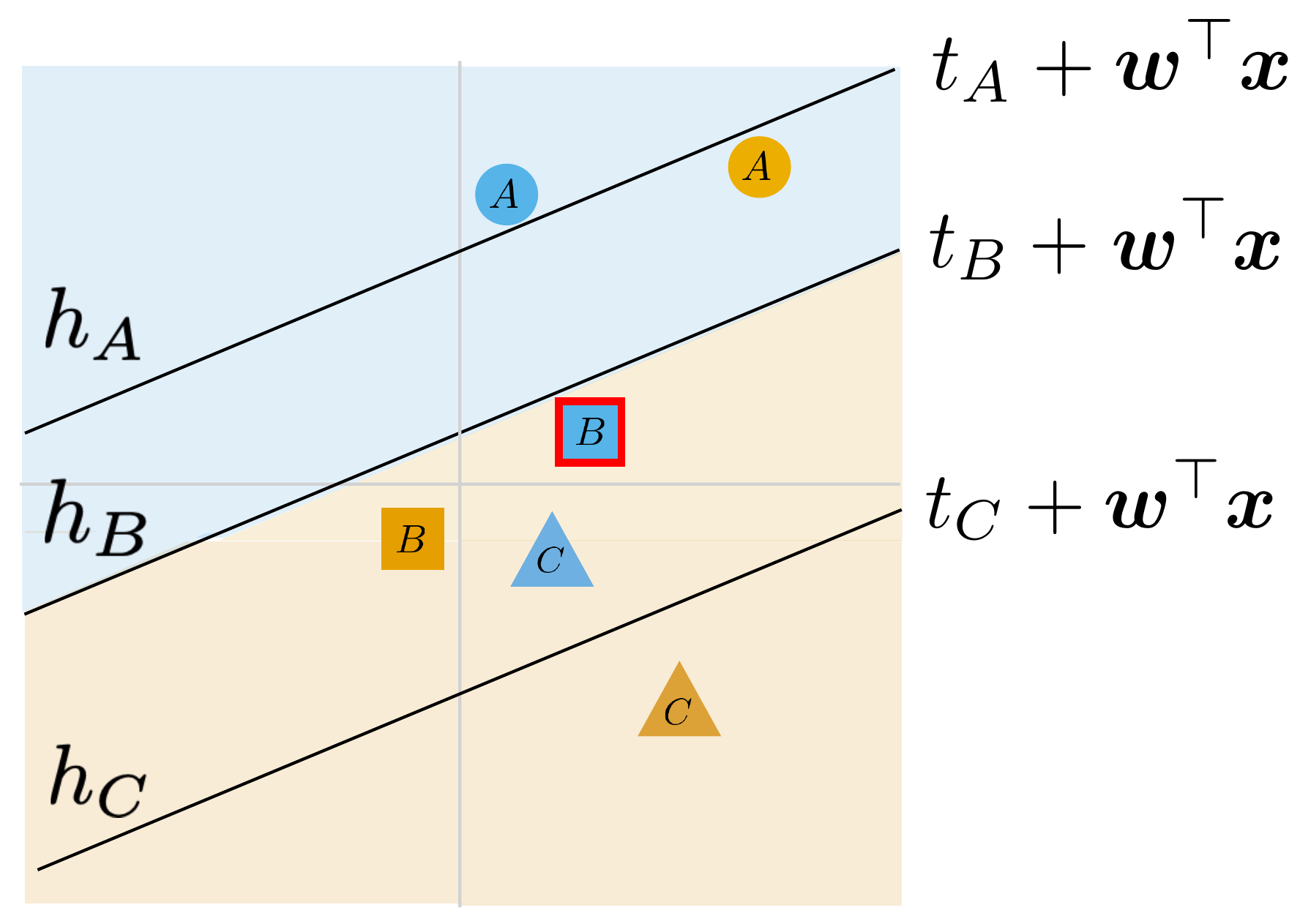} &
\addplotgen{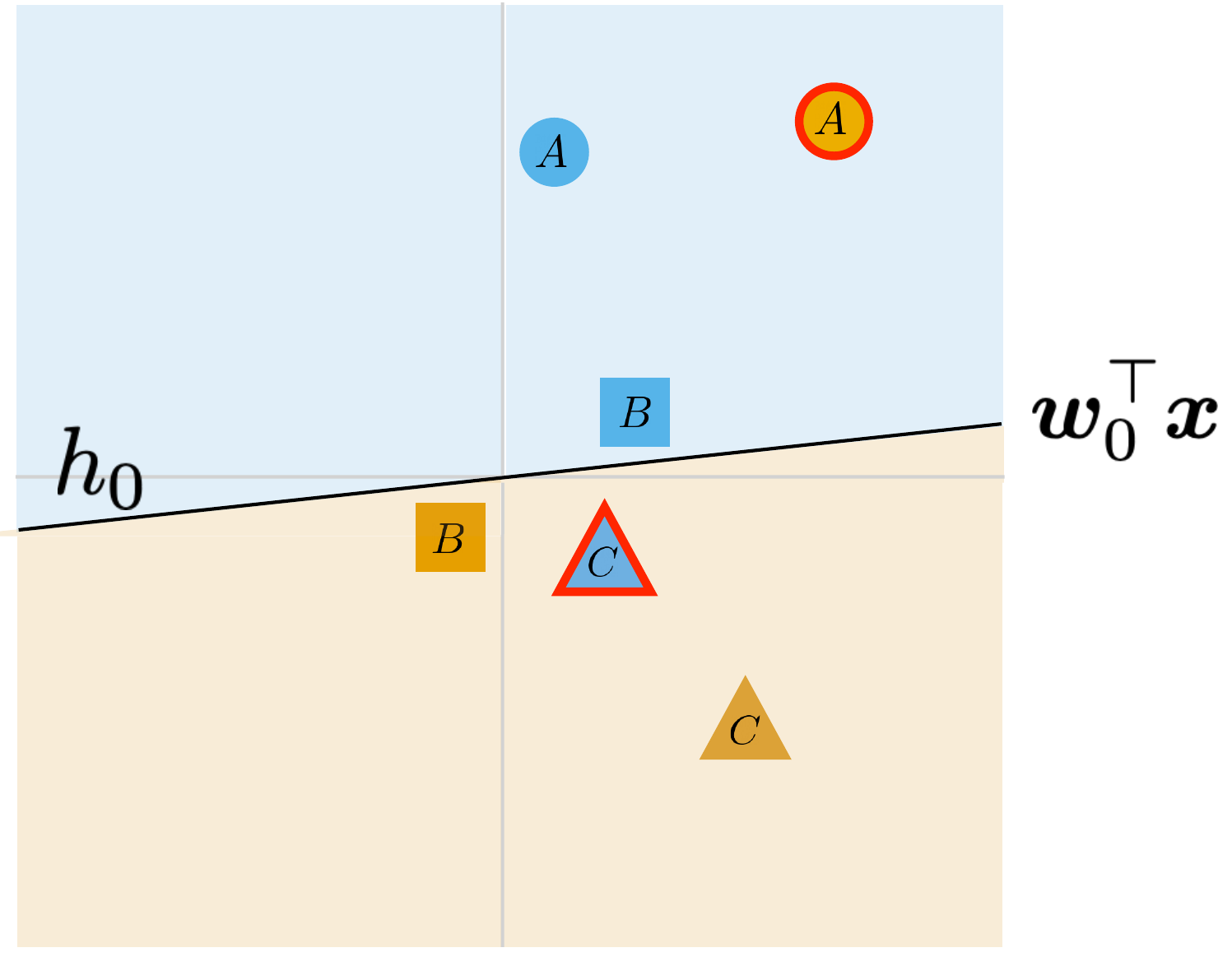}
\end{tabular}
    \caption{ERM returns a misspecified personalized model that assigns a personalized intercept for each group but the same slope for all groups. It does not capture the personalized slope needed to accurately model group B. The model improves overall performance by assigning more accurate predictions to groups $A$ and $C$. However, it performs \emph{worse} for group $B$.}
    \label{Fig::FMGroupSpecificEffects}
\end{figure*}

\paragraph{Misspecification}

We start with misspecification -- i.e., when a model that cannot capture the influence of group membership in a conditional data distribution. A common form of misspecification occurs when we personalize simple models with a one-hot encoding~\citep{van2022harm}. In such cases, models exhibit fair use violations on data distributions that exhibit \emph{intersectionality}. Consider, for example, a logistic regression model with a one-hot encoding that assigns higher risk to patients who are \textgp{young}, and to patients who are \textgp{female}. This model would exhibit a fair use violation for patients who are \textgp{young} \emph{and} \textgp{female} if their true risk were lower due to an interaction effect among group attributes (see \cref{Fig::FailureModeMisspecification}). 

Misspecification can also stem from group-specific interaction effects -- e.g., tasks where group attributes act as mediators or moderators~\citep[see e.g.,][]{baron1986moderator}. In \cref{Example::FMGroupSpecificEffects}, we show an example that exhibits the hallmarks of personalization: a generic model performs poorly on ``heterogeneous" groups $A$ and $C$, and a personalized model that targets these groups improves performance at a population-level. 
\begin{example}%[Group Specific Interaction Effects]
\label{Example::FMGroupSpecificEffects}
Consider a 2D classification task with groups $\G = \{A, B, C\}$ with 1 positive and 1 negative example in which a Bayes optimal classifier $\plf{}:\X \times \G \to \Y$ should assign a personalized intercept to each group \emph{and} a personalized slope to group $B$:
\begin{align*}
    h(\xb,\gb) = 
    \begin{cases}
    \sign{t_A + \bm{w}^\top \xb} & \textrm{if}~ \gb = A\\
    \sign{t_B + \bm{w}_{B}^\top \xb} & \textrm{if}~ \gb = B\\
    \sign{t_C + \bm{w}^\top\xb}
    & \textrm{if}~ \gb = C
    \end{cases}
\end{align*}

Here, ERM with a standard one-hot encoding of $\G$ would return a personalized model that assigns a personalized intercept for each group, but the same slope to all three groups:
\begin{align*}
    h(\xb,\gb) = 
    \begin{cases}
    \sign{t_A + \bm{w}^\top \xb} & \textrm{if}~ \gb = A\\
    \sign{t_B + \bm{w}^\top \xb} & \textrm{if}~ \gb = B\\
    \sign{t_C + \bm{w}^\top\xb}
    & \textrm{if}~ \gb = C
    \end{cases}
\end{align*}
The model would improve overall performance by assigning more accurate predictions to groups $A$ and $C$. However, it would perform \emph{worse} for group $B$ (see \cref{Fig::FMGroupSpecificEffects}).
\end{example}
Resolving violations from model misspecification is difficult since it requires interventions that can resolve them for all groups. In practice, one could fit models from a class that is rich enough to capture these effects, or train a separate model for each group. Both approaches are challenging when working with multiple groups. The first requires that we either specify interactions for each group and fit these terms correctly. The second requires that we train models using a limited amount of data for each group. 
 
\paragraph{Model Selection}

Model development often involves choosing a model from candidate models -- e.g., when setting a regularization penalty to avoid overfitting or to induce sparsity. Common criteria for model selection guide these decisions on the basis of population-level performance~\citep[e.g., mean K-CV test error][]{arlot2010survey}. As shown in \cref{Fig::FailureModeFeatureSelection}, the resulting model may improve performance for one group while reducing performance for another group in tasks with heterogeneous data distributions.

\begin{example}%[Fair Use Violation in Model Selection]
\label{Fig::FailureModeFeatureSelection}
Consider a classification task where a personalized model must use either $x_1 \in \{0,1\}$ or $x_2 \in \{0,1\}$. We are given 60 examples from group $A$ and 90 examples from group $B$. We train a personalized model with a one-hot encoding of $\G = \{A,B\}$ choosing between $x_1$ or $x_2$ to minimize the overall error rate.
\begin{center}\normalfont
\setlength{\tablewidth}{\iftoggle{arxiv}{0.8\textwidth}{\textwidth}}
\renewcommand{\ftblmidrules}[0]{%
    \cmidrule(lr){1-4}
    \cmidrule(lr){5-6}
    \cmidrule(lr){7-9}
    \cmidrule(lr){10-12}
}
\resizebox{\tablewidth}{!}{\begin{tabular}{rcrrrrrrrccc}
        \ftblheader{4}{c}{} &
        \ftblheader{2}{c}{Generic} 
        & \ftblheader{3}{c}{Personalized with $x_1$} &   
        \ftblheader{3}{c}{Personalized with $x_2$} \\ 
        %(Discarded)\\$h_{D}(x_{2}, g)$}} \\
        %
        \cmidrule(lr){5-6}\cmidrule(lr){7-9}\cmidrule(lr){10-12}
        % & Data & & Pred. & Mistakes &
        %  Pred. & Error &
        %  Gain & Pred. & Error &
        %  Gain \\
        %  \ftblmidrules{}
        %        %
        Group & $(x_{1}, x_{2})$ & $\nplus{}$ & $\nminus{}$ & 
        $\clf{}$ & $R(\clf{})$ & 
        $h_{1}$ & $R(h_{1})$ & $\Delta$ & 
        $h_{2}$ & $R(h_{2})$ & $\Delta$ \\ 
        \ftblmidrules{}
         $A$ & $(0, 0)$ & $10$ &
         $0$ &
         $+$ & $0$ &
         $+$ & $0$ & $0$ &
         $+$ & $10$ & $-10$ \\
         $A$ & $(0, 1)$ & $10$ &
         $0$ &
         $+$ & $0$ &
         $+$ & $0$ & $0$ &
         $+$ & $0$ & $0$ \\
         $A$ & $(1, 0)$ & $0$ &
         $20$ &
         $+$ & $20$ &
         $-$ & $0$ & $20$ &
         $-$ & $0$ & $20$ \\
         $A$ & $(1, 1)$ & $20$ &
         $0$ &
         $+$ & $0$ &
         $-$ & $20$ & $-20$ &
         $+$ & $0$ & $0$ \\
         
        \cmidrule(lr){2-4}
        \cmidrule(lr){5-6}
        \cmidrule(lr){7-9}
        \cmidrule(lr){10-12}
         
         $B$ & $(0, 0)$ & $5$ &
         $0$ &
         $+$ & $0$ &
         $-$ & $5$ & $-5$ &
         $+$ & $0$ & $0$ \\
         $B$ & $(0, 1)$ & $0$ &
         $20$ &
         $+$ & $20$ &
         $-$ & $0$ & $20$ &
         $+$ & $20$ & $0$ \\
         $B$ & $(1, 0)$ & $20$ &
         $0$ &
         $+$ & $0$ &
         $+$ & $0$ & $0$ &
         $+$ & $0$ & $0$ \\
         $B$ & $(1, 1)$ & $30$ &
         $0$ &
         $+$ & $0$ &
         $+$ & $0$ & $0$ &
         $+$ & $0$ & $0$ \\
        
        \ftblmidrules{}
         {\bfseries Total} & & $95$ & $40$
          &  & $40$ & 
          & $25$ & \cellcolor{fgood!60} $15$
          & & $30$ & \cellcolor{fgood!60} $10$ \\
        {\bfseries Group $A$} & & $40$ & $20$
          &  & $20$ & 
          & $20$ & \cellcolor{funk!60} $0$
          & & $10$ & \cellcolor{fgood!60} $10$ \\
        {\bfseries Group $B$} & & $55$ & $20$
          &  & $20$ & 
          & $5$ & \cellcolor{fgood!60} $15$
          & & $20$ & \cellcolor{funk!60} $0$ 
    \end{tabular}}
\end{center}
The generic model $\clf$ is the same whether it uses $x_1$ or $x_2$. However, a personalized model would violate fair use for group $A$ if it uses $x_1$, and would violate fair use for group $B$ if it uses $x_2$. In this case, ERM returns the personalized model that benefits group $A$ -- i.e., the majority group.
\end{example}
\cref{Fig::FailureModeFeatureSelection} could arise, for example, when developing a clinical prediction model using features that encode the outcome of competing diagnostics. More broadly, \cref{Fig::FailureModeFeatureSelection} highlights how fair use violations may be unavoidable when we must assign predictions with a single model -- as the task shows that models trained with $x_1$ and $x_2$ would lead to fair use violations on $A$ or $B$ respectively.

\subsection{Sufficient Conditions}
\newcommand{\VC}[1]{D}

% STATEMENT

We present sufficient conditions for ERM with group attributes to output a model that obeys fair use in training (Proposition~\ref{Prop::TrainingGuarantee}) and testing (Proposition~\ref{Prop::TestingGuarantee}).
\begin{proposition}
\label{Prop::TrainingGuarantee} 
Consider training a personalized model by ERM $\plf{} \in \argmin_{h \in \Hset}\emprisk{}{h}$, and evaluating its gains to personalization with respect to a generic model $\clf{} \in \argmin_{h \in \Hset_0}\emprisk{}{h}$ where $\Hset_0 \subseteq \Hset.$ The personalized model $\plf{}$ obeys fair use in terms of empirical risk so long as the model achieves the same risk as a model that specifically targets the group. That is: $$\emprisk{\gb}{\plf{}} = \emprisk{\gb}{\plf{\gb}} ~\textrm{for all groups}~\gb \in \G.$$
\end{proposition}
\cref{Prop::TrainingGuarantee} holds for settings where we fit personalized models from a class $\Hset$ that extends the generic model class $\Hset_0$ (see~\cref{Def:Extends}). This requirement implies that we should fit personalized models from model classes that are rich enough to target each intersectional group. When we personalize a linear classifier via ``score correction"~\citep[][]{van2022harm}, we should include a correction term for each group. Otherwise, we may violate fair use due to model misspecification when using a one-hot encoding as in \cref{Fig::FailureModeMisspecification}. Likewise, if we personalize a model with interaction terms, we should include an interaction for each group.
More broadly, the conditions in \cref{Prop::TrainingGuarantee} are met when, for example, we use the data from each group to train a model for each group. Given that these are sufficient conditions, it is still possible to achieve fair use even when they don't hold.

\begin{proposition}
\label{Prop::TestingGuarantee}
Consider a personalized model $\plf{}: \X \times \G \to \Y$ that ensures rationality and envy-freeness for group $\gb$ in terms of empirical risk. Denote the empirical gains in rationality and envy-freeness for group $\gb$ as:
\begin{align*}
\empgapmin{\gb} := \empgap{\gb}{\plf{\gb}}{\clf{}}, \qquad \hat{\gamma}_{\gb} := \min_{\gb' \in \G / \{\gb\}} \empgap{\gb}{\plf{\gb}}{\plf{\gb'}}.
\end{align*}
If $\empgapmin{\gb} > 0$, then rationality for group $\gb$ generalizes with probability at least $1 - \delta$ as long as:
\[
 n_{\gb} \geq \frac{4\VC{\Hset} \log \left(\frac{2n_{\gb}}{\VC{\Hset}} + 1\right) + \log \left(\tfrac{8}{\delta}\right)}{\empgapmin{\gb}^2}
\]
If $\hat{\gamma}_{\gb} > 0$, then envy-freeness for group $\gb$ generalizes with probability at least $1 - \delta$ as long as:
\begin{align*}
 n_{\gb} &\geq \frac{4\VC{\Hset} \log \left(\frac{2n_{\gb}}{\VC{\Hset}} + 1\right) + \log \left(\tfrac{8m}{\delta}\right)}{\hat{\gamma}_{\gb}^2}
\end{align*}
\end{proposition}
\cref{Prop::TestingGuarantee} characterizes the sample complexity of generalization for personalized models that satisfy fair use conditions on training data. The bounds apply to a general class of personalized models, and can be strengthened by assuming a finite hypothesis class~\citep[e.g, in][]{ustun2019fairness}, or by accounting for distributional differences between groups~\citep[e.g.,][]{pmlr-v97-wang19l}.
The result holds in tasks where personalization leads to strictly positive gains with respect to rationality and envy-freeness on the training data, which is not guaranteed in practice and must be checked in practice.

\section{Empirical Study}
\label{Sec::Experiments}

% \paragraph{Methods}
% commands for method names here
\newcommand{\LR}[0]{\textmn{LR}}
\newcommand{\RF}[0]{\textmn{RF}}
\newcommand{\NN}[0]{\textmn{NN}}

\newcommand{\onehot}[0]{\textmn{1Hot}}
\newcommand{\intonehot}[0]{\textmn{All}}
\newcommand{\dcp}[0]{\textmn{DCP}}

\newcommand{\titlecell}[2]{\setlength{\tabcolsep}{0pt}{\footnotesize{\textsc{\begin{tabular}{#1}#2\end{tabular}}}}}
\newcommand{\bfcell}[2]{\setlength{\tabcolsep}{0pt}\textbf{\begin{tabular}{#1}#2\end{tabular}}}
\newcommand{\rangecell}[3]{\cell{c}{{#1}\\{\color{gray}\small{#2} -- {#3}}}}

\newcommand{\tablemetric}{Generic}
\newcommand{\tablepersonalized}{Personalized}
\newcommand{\metricsguide}[0]{\cell{r}{Personalized\\Gain\\Best/Worst Gain\\Rat. Gains/Viols \\EF Gains/Viols}}

\newcommand{\adddatainfo}[6]{\cell{l}{%
\textds{#1}\\%
$n={#2}, d={#3}$\\
$\G=\{#5\}$\\
$m={#4}$\\
{#6}
}}

% \adddatainfo{data_name}{n}{d}{m}{group attributes}{anything else}
\newcommand{\apnea}[0]{\adddatainfo{apnea}{1152}{26}{6}{\textfn{age},\textfn{sex}}{\citet[][]{ustun2016clinical}}}
\newcommand{\cshocke}[0]{\adddatainfo{cardio\_eicu}{1341}{49}{4}{\textfn{age}, \textfn{sex}}{\citet[][]{pollard2018eicu}}}
\newcommand{\cshockm}[0]{\adddatainfo{cardio\_mimic}{5289}{49}{4}{\textfn{age},\textfn{sex}}{\citet[][]{johnson2016mimic}}}
\newcommand{\heart}[0]{\adddatainfo{heart}{181}{26}{4}{\textfn{sex}, \textfn{age}}{\citet[][]{detrano1989international}}}
\newcommand{\kidney}[0]{\adddatainfo{kidney}{2066}{78}{6}{\textfn{sex}, \textfn{ethnicity}}{\citet[][]{zhang2021learning}}}
\newcommand{\mortality}[0]{\adddatainfo{mortality}{25366}{468}{6}{\textfn{age},\textfn{sex}}{\citet[][]{johnson2016mimic}}}
\newcommand{\saps}[0]{\adddatainfo{saps}{7797}{36}{4}{\textfn{hiv},\textfn{age}}{\citet[][]{allyn2016simplified}}}

In this section, we present an empirical study of fair use in clinical prediction models -- i.e. a class of models that routinely include group attributes and where fair use violations inflict harm. Our goals are to discuss the prevalence of fair use violations, the impact of standard personalization techniques, and the potential to resolve them through interventions in model development. %
We provide code to reproduce these results at \url{https://github.com/ustunb/fairuse} and 
include additional results in~\cref{Appendix::Experiments}.

\begin{table*}[t]
\centering
\begingroup
\resizebox{\linewidth}{!}{
\begin{tabular}{@{}l@{}r@{}H@{}H@{}H@{}*{9}c}
%\toprule
\multicolumn{2}{l}{ } & \multicolumn{3}{c}{} & \multicolumn{3}{c}{Test Error} & \multicolumn{3}{c}{Test AUC} & \multicolumn{3}{c}{Test ECE} \\
%\cmidrule(l{3pt}r{3pt}){3-5} 
\cmidrule(l{-2pt}r{3pt}){6-8} \cmidrule(l{3pt}r{3pt}){9-11} \cmidrule(l{3pt}r{3pt}){12-14}

Dataset & Metrics & 
\onehot{} & \intonehot{} & \dcp{}  & 
\onehot{} & \intonehot{} & \dcp{}  & 
\onehot{} & \intonehot{} & \dcp{} & 
\onehot{} & \intonehot{} & \dcp{}\\

\cmidrule(l{-2pt}r{-2pt}){1-2} \cmidrule(l{-2pt}r{3pt}){6-8} \cmidrule(l{3pt}r{3pt}){9-11} \cmidrule(l{3pt}r{3pt}){12-14}

\apnea{} & \metricsguide{} & \cell{r}{0.589\\0.001\\0.009 / -0.000\\3/3\\5/3} & \cell{r}{0.588\\0.001\\0.002 / -0.004\\3/3\\5/5} & \cell{r}{0.681\\-0.092\\0.063 / -3.440\\3/3\\2/2} & \cell{r}{34.2\%\\-1.0\%\\0.0\% / -9.6\%\\6/4\\0/0} & \cell{r}{33.8\%\\-0.7\%\\1.7\% / -7.8\%\\5/3\\1/0} & \cell{r}{26.2\%\\7.0\%\\21.7\% / -7.8\%\\2/2\\5/4} & \cell{r}{0.750\\0.001\\0.002 / -0.001\\1/4\\0/6} & \cell{r}{0.750\\0.000\\0.001 / -0.010\\1/2\\0/6} & \cell{r}{0.803\\0.053\\0.119 / -0.005\\4/4\\0/0} & \cell{r}{7.5\%\\-1.5\%\\0.9\% / -8.6\%\\3/3\\3/0} & \cell{r}{5.5\%\\0.6\%\\0.8\% / -4.6\%\\3/3\\0/0} & \cell{r}{7.2\%\\-1.1\%\\1.7\% / -6.6\%\\3/3\\4/4}\\

\cmidrule(l{-2pt}r{-2pt}){1-2} \cmidrule(l{-2pt}r{3pt}){6-8} \cmidrule(l{3pt}r{3pt}){9-11} \cmidrule(l{3pt}r{3pt}){12-14}

\cshocke{} & \metricsguide{} & \cell{r}{0.591\\-0.000\\0.000 / -0.000\\2/2\\2/2} & \cell{r}{0.591\\-0.000\\0.001 / -0.001\\3/3\\1/1} & \cell{r}{1.045\\-0.454\\0.092 / -5.470\\1/1\\1/1} & \cell{r}{29.1\%\\-0.4\%\\0.0\% / -3.1\%\\4/2\\1/0} & \cell{r}{29.1\%\\-0.4\%\\0.2\% / -3.1\%\\4/2\\1/0} & \cell{r}{29.5\%\\-0.9\%\\13.0\% / -8.6\%\\2/2\\3/3} & \cell{r}{0.768\\0.000\\0.002 / -0.001\\2/3\\0/4} & \cell{r}{0.767\\-0.001\\0.001 / -0.001\\2/3\\0/4} & \cell{r}{0.762\\-0.007\\0.096 / -0.104\\1/1\\1/1} & \cell{r}{4.4\%\\0.4\%\\1.6\% / -1.5\%\\1/1\\0/0} & \cell{r}{4.6\%\\0.2\%\\0.9\% / -0.2\%\\0/0\\0/0} & \cell{r}{8.9\%\\-4.1\%\\-0.9\% / -6.2\%\\4/4\\1/1}\\

\cmidrule(l{-2pt}r{-2pt}){1-2} \cmidrule(l{-2pt}r{3pt}){6-8} \cmidrule(l{3pt}r{3pt}){9-11} \cmidrule(l{3pt}r{3pt}){12-14}

\cshockm{} & \metricsguide{} & \cell{r}{0.474\\0.001\\0.003 / 0.000\\4/4\\1/1} & \cell{r}{0.474\\0.001\\0.003 / -0.000\\3/3\\1/1} & \cell{r}{0.452\\0.023\\0.092 / 0.002\\4/4\\0/0} & \cell{r}{23.3\%\\0.3\%\\0.9\% / -0.1\%\\1/0\\1/0} & \cell{r}{23.4\%\\0.3\%\\0.9\% / -0.1\%\\1/0\\1/0} & \cell{r}{21.4\%\\2.2\%\\7.9\% / -0.0\%\\0/0\\4/4} & \cell{r}{0.854\\0.001\\0.001 / -0.000\\2/2\\0/4} & \cell{r}{0.854\\0.001\\0.001 / -0.000\\2/2\\0/4} & \cell{r}{0.870\\0.017\\0.053 / 0.006\\4/4\\0/0} & \cell{r}{2.1\%\\-0.4\%\\0.5\% / 0.4\%\\0/0\\1/1} & \cell{r}{2.3\%\\-0.5\%\\0.6\% / -0.2\%\\0/0\\0/0} & \cell{r}{2.3\%\\-0.6\%\\0.8\% / -2.3\%\\2/2\\3/3}\\

\cmidrule(l{-2pt}r{-2pt}){1-2} \cmidrule(l{-2pt}r{3pt}){6-8} \cmidrule(l{3pt}r{3pt}){9-11} \cmidrule(l{3pt}r{3pt}){12-14}

\heart{} & \metricsguide{} & \cell{r}{0.590\\0.015\\0.309 / -0.139\\2/2\\2/2} & \cell{r}{0.722\\-0.117\\0.040 / -0.579\\1/1\\3/3} & \cell{r}{3.041\\-2.436\\0.012 / -7.401\\0/0\\2/2} & \cell{r}{19.7\%\\-1.3\%\\0.0\% / -6.8\%\\4/1\\3/0} & \cell{r}{19.7\%\\-1.3\%\\0.1\% / -9.9\%\\4/1\\3/0} & \cell{r}{15.8\%\\2.6\%\\10.6\% / -8.4\%\\2/1\\2/2} & \cell{r}{0.870\\-0.007\\0.008 / -0.036\\1/3\\0/4} & \cell{r}{0.846\\-0.030\\0.017 / -0.055\\0/3\\0/4} & \cell{r}{0.817\\-0.060\\0.039 / -0.190\\1/1\\2/2} & \cell{r}{8.4\%\\2.8\%\\3.7\% / -0.5\%\\1/1\\1/1} & \cell{r}{17.8\%\\-6.6\%\\-1.2\% / -3.2\%\\3/3\\0/0} & \cell{r}{17.5\%\\-6.3\%\\10.1\% / -4.6\%\\1/1\\2/2}\\

\cmidrule(l{-2pt}r{-2pt}){1-2} \cmidrule(l{-2pt}r{3pt}){6-8} \cmidrule(l{3pt}r{3pt}){9-11} \cmidrule(l{3pt}r{3pt}){12-14}

\mortality{} & \metricsguide{} & \cell{r}{0.482\\0.001\\0.006 / -0.001\\4/4\\6/2} & \cell{r}{0.482\\0.001\\0.007 / -0.000\\4/4\\2/2} & \cell{r}{0.436\\0.047\\0.422 / 0.016\\6/6\\0/0} & \cell{r}{23.6\%\\-0.2\%\\0.8\% / -2.5\%\\4/4\\2/0} & \cell{r}{23.4\%\\0.0\%\\2.1\% / -0.4\%\\2/2\\2/1} & \cell{r}{20.2\%\\3.2\%\\20.8\% / -0.6\%\\1/1\\6/6} & \cell{r}{0.848\\0.000\\0.004 / -0.001\\3/3\\0/6} & \cell{r}{0.848\\0.001\\0.004 / -0.000\\4/4\\0/6} & \cell{r}{0.880\\0.033\\0.114 / 0.011\\6/6\\0/0} & \cell{r}{2.0\%\\0.2\%\\1.6\% / 0.0\%\\0/0\\4/0} & \cell{r}{2.1\%\\0.1\%\\2.9\% / -0.5\%\\0/0\\3/3} & \cell{r}{2.5\%\\-0.3\%\\11.2\% / -2.5\%\\3/3\\5/5}\\

\cmidrule(l{-2pt}r{-2pt}){1-2} \cmidrule(l{-2pt}r{3pt}){6-8} \cmidrule(l{3pt}r{3pt}){9-11} \cmidrule(l{3pt}r{3pt}){12-14}

\saps{} & \metricsguide{} & \cell{r}{0.417\\0.002\\0.115 / 0.000\\4/4\\0/0} & \cell{r}{0.417\\0.002\\0.113 / 0.000\\4/4\\2/2} & \cell{r}{0.497\\-0.078\\0.017 / -16.301\\2/2\\2/2} & \cell{r}{18.9\%\\0.0\%\\16.4\% / -12.2\%\\2/2\\2/1} & \cell{r}{18.9\%\\0.0\%\\0.7\% / -12.2\%\\3/2\\2/2} & \cell{r}{18.5\%\\0.4\%\\3.5\% / -23.3\%\\2/1\\2/2} & \cell{r}{0.890\\0.001\\0.013 / -0.000\\1/3\\0/4} & \cell{r}{0.890\\0.001\\0.013 / -0.000\\1/3\\0/4} & \cell{r}{0.888\\-0.001\\0.017 / -0.246\\2/2\\1/2} & \cell{r}{1.6\%\\0.0\%\\2.9\% / -2.1\%\\2/2\\2/2} & \cell{r}{1.6\%\\0.0\%\\2.5\% / -1.3\%\\2/2\\2/2} & \cell{r}{1.9\%\\-0.3\%\\9.4\% / -19.1\%\\2/2\\3/3}\\

\cmidrule(l{-2pt}r{-2pt}){1-2} \cmidrule(l{-2pt}r{3pt}){6-8} \cmidrule(l{3pt}r{3pt}){9-11} \cmidrule(l{3pt}r{3pt}){12-14}

\end{tabular}
}
\endgroup
\caption{Performance of personalized logistic regression models on all datasets. We show the gains of personalization in terms of test AUC, ECE, and error. We report: model performance at the population level, the overall gain of personalization, the range of gains over $m$ intersectional groups, and the number of rationality and envy-freeness gains/violations (evaluated using a bootstrap hypothesis test (\cref{Sec::HypothesisTest}) at a 10\% significance level). We include results for other model classes in \cref{Appendix::Experiments}.}
\label{Table::Results}
\end{table*}

\subsection{Setup}
\label{Sec::ExperimentalSetup}

%\paragraph{Datasets}

We work with 6 datasets for clinical prediction tasks listed in \cref{Table::Results} and \cref{Appendix::Datasets}. We minimally process each dataset to impute the values of missing points (using mean value imputation), and repair class imbalances across intersectional groups (to eliminate ``trivial" fair use violations that occur due to class imbalance). We split each dataset into a training sample (80\%) to fit models, and a test sample (20\%) to evaluate the gains of personalization. 

%\paragraph{Personalized Models}

We train 9 personalized models for each dataset. Each model belongs to one of 3 model classes: \emph{logistic regression} (\LR{}), \emph{random forests} (\RF{}), and \emph{neural nets} (\NN{}), and encodes group attributes using one of 3 personalization techniques:
\begin{itemize}[leftmargin=0pt, label={},itemsep=0.15em, topsep=0.0em]

\item \emph{One-Hot Encoding} (\onehot{}): We train a model with features that include dummy variables for each group attribute.

\item \emph{Intersectional Encoding} (\intonehot{}): We train with features that include dummy variables for each intersectional group.

\item \emph{Decoupling} (\dcp{}): We train a separate model for each intersectional group using only data from this group $\gb_i = \gb$. 
\end{itemize}
These three techniques reflect the increasingly complex approaches available to practitioners to account for group membership in a prediction model as measured in terms of the interactions between group attributes and other features: \onehot{} reflect no interactions; \intonehot{} reflect interactions between group attributes; and \dcp{} reflects all possible interactions between group attributes and features.

We evaluate the gains of personalization for each model in terms of three performance metrics, reflecting common metrics that are encountered in different tasks: (1) \emph{error rate}, which reflects the accuracy of yes-or-no predictions, e.g., for a diagnostic test~\citep[][]{eusebi2013diagnostic}; (2) \emph{area under ROC curve} (AUC), which measures accuracy in ranking, e.g., for triage~\citep[e.g.,][]{zhan2020blastocyst}; (3) \emph{expected calibration error} (ECE), which measures the reliability of risk predictions for a risk score~\citep[][]{blaha2016critical,ustun2019learning}.

\subsection{Results}
\label{Sec::ExperimentalResults}

We summarize our results for logistic regression in \cref{Table::Results} and for neural networks and random forests in~\cref{Appendix::Experiments}. 

\paragraph{On the Prevalence of Fair Use Violations}

Our results show that we train models that improve population level performance across prediction tasks in terms of training loss (guaranteed), training performance (expected), and test performance (expected). Yet personalized models that improve performance at a population level can also reduce performance for specific groups. These violations arise across datasets, personalization techniques, and model classes.

We consider the standard configuration used to develop clinical prediction models -- i.e., a logistic regression model with a one-hot encoding of group attributes (\LR{}+\onehot{}). In this case, we find that at least one group experiences a statistically significant fair use violation in terms of error on 4/6 datasets (5/6 for AUC and ECE). On \textds{saps}, for example, \LR{} + \onehot{} exhibits a statistically significant gain from personalization for patients over 30 who are HIV negative. Conversely, in \textds{cardio\_eicu} when training \LR{}+\intonehot{} we detect a fair use violation for old females (see e.g., 4/2 Rat. Gains/Viols. respectively for test error in \cref{Table::Results}).

\paragraph{On the Robustness of Personalization Techniques}

Our results show there is no one personalization technique that can avoid fair use violations, as demonstrated by the fact that the personalization technique that minimizes fair use violations varies across datasets, model classes, and prediction tasks.   
In \cref{Table::Results}, for example, we find that the best technique to personalize a logistic regression model for \textds{cardio\_eicu} is to use an intersectional encoding, but to train decoupled models for \textds{mortality}. These strategies change across model classes -- as the ideal strategies for neural networks are decoupling and intersectional encoding, respectively \textds{cardio\_eicu} and \textds{mortality} (see \cref{Appendix::Experiments}). %
Even configurations that exhibit few violations across datasets may fail critically across groups. For example, \LR{}+\dcp{} for \textds{saps} leads to a 10\% increase in error for \textfn{HIV+ \& >30}. 
Overall, these results suggest that the most reliable way to avoid a fair use violation is to check.

\paragraph{On Detecting Violations}

Our results underscore the need for reliable procedures to spot fair use violations or claim gains from personalization. We can often find reliable instances of benefit or harm but we sometimes are unable to do so. An actionable finding from evaluating the gains of personalization is a group does not experience a meaningful gain nor harm due to personalization. We note a number of cases across datasets, personalization techniques, and model classes where we note no meaningful gain or harm. Often times this is because the effect size is small or the group sample sizes are too small. 

In such cases where we are unable to detect any impact from personalizationo, one may wish to intervene to avoid soliciting unnecessary data. For example, when group attributes encode information that is sensitive or must be collected at prediction time (e.g., \textfn{HIV}), we may prefer to avoid soliciting information unless it is demonstrably useful for prediction.

\paragraph{On Resolving Violations}

Our results show that routine decisions in model development can induce considerable differences in group-level performance. This suggests that we can reduce fair use violations through ``interventions" in model development. We studied the effectiveness of this approach through an ablation study where we repeated our experiments with interventions that address failure modes in \cref{Sec::FailureModes}, namely: using an intersectional one-hot encoding, decoupled training, and equalizing sample sizes. 

Our results show that interventions can often reduce fair use violations. For example, we can eliminate all fair use violations for \textds{cardio\_mimic} in our standard configuration by decoupled training. However, there is no ``silver bullet" intervention that resolves fair use violations across all datasets and model classes. In general, the best intervention varies across model classes and datasets. In some cases, the best intervention may fail to resolve all fair use violations as resolving a violation for one group may induce a violation on another group. In \textds{cardio\_eicu}, for example, a logistic regression model with a one-hot encoding will exhibit a violation on old males. Switching an intersectional encoding will fix this violation but introduce another for old females.

%These results lead us to explore other kinds of interventions in \cref{Sec::Demonstration}.

\begin{table*}[!t]
    \renewcommand{\ftblmidrules}[0]{\cmidrule(lr){1-1}\cmidrule(lr){2-4}\cmidrule(lr){5-6}}
    \centering
    \resizebox{\linewidth}{!}{
    %
    %%%%%%%%%%%%%%%%%%%%%%%%%%%%%%%%%%%%%%%%
    %     \renewcommand{\ftblmidrules}[0]{%
    %     \cmidrule(lr){1-3}
    %     \cmidrule(lr){4-5}
    % }
     \begin{tabular}{rHcccc}
    \multicolumn{1}{c}{Group} &
     \multicolumn{3}{c}{\textsc{Test Error}} & \multicolumn{2}{c}{\textsc{Intervention}} \\
    \cmidrule(lr){1-1}\cmidrule(lr){2-4}\cmidrule(lr){5-6}    
\multicolumn{1}{c}{$\gb{}$} & $\truerisk{}{\clf{}}$ & $ \truerisk{\gb}{\plf{\gb}}$ &  $\Delta_{\gb}$ & 
      \cell{r}{Assign $\clf{}$} & 
    \cell{r}{Assign $h^\textrm{dcp}_{\gb}$} \\
    \ftblmidrules{}
                  \textfn{female}, \textfn{black} & 59.0\% &                    55.5\% &        \good{3.5\%} & 3.5\% & \good{ 33.1\%} \\   
                  \textfn{female}, \textfn{white} & 23.9\% &                    21.9\% &        \good{2.0\%} & 2.0\% & 2.0\% \\    
                 \textfn{female}, \textfn{other} & 27.0\% &                    20.4\% &        \good{6.6\%} & 6.6\% & \good{ 9.1\%} \\        
              \textfn{male}, \textfn{black} & 26.7\% &                    29.4\% &       \violation{-2.7\%} & \closecall{0.0\%} & \good{ 15.6\%} \\            
              \textfn{male}, \textfn{white} & 30.6\% &                    21.9\% &       \good{8.1}\% & 8.1\% & \good{ 3.7\%} \\            
                \textfn{male}, \textfn{other} & 23.4\% &                    25.3\% &        \violation{-1.9\%} & \closecall{0.0\%} & \good{ 1.3\%} \\            
    \ftblmidrules{}
                  \textbf{Total} & \color{black}{28.5\%} &                    \color{black}{27.1\%} &       \good{1.4\%} & - & - 
                  %\\ \ftblmidrules{}
    \end{tabular}

%     %%%%%%%%%%%%%%%%%%%%%%%%%%%%%%%%%%%%%%%%%
    \renewcommand{\ftblmidrules}[0]{%
    \cmidrule(lr){1-3}\cmidrule(lr){4-5}%
    }
\begin{tabular}{Hcccc}
    \multicolumn{3}{c}{\textsc{Test AUC}}  & \multicolumn{2}{c}{\textsc{Interventions}}\\ \ftblmidrules{}
    $\truerisk{}{\clf{}}$ &  $ \truerisk{\gb}{\plf{\gb}}$ &  $\Delta_{\gb}$ & 
    \cell{r}{Assign $\clf{}$} & 
    \cell{r}{Assign $h^\textrm{dcp}_{\gb}$}\\
    %$\gb$ & %$n$ &  %$h_0(\xb) &  %$h(\xb, \gb) &  %$\Remp{\gb}h(\xb, \gb)) - \Remp{\gb}(h_0(\xb))$\\
    \ftblmidrules{}
    0.433 &                    0.443 &        \good{0.010} & 0.010 & \good{0.315} \\ 
     0.841 &                    0.845 &        \good{0.004} & 0.004 & \good{0.057}  \\    
     0.864 &                    0.861 &       \violation{-0.003} & \closecall{0.000} & \good{0.038}  \\        
     0.779 &                    0.799 &        \good{0.020} & 0.020 & \good{0.096} \\
     0.761 &                    0.767 &       \good{0.006} & 0.006 & \good{0.104}     \\  
    0.838 &                    0.835 &        \violation{-0.003} &  \closecall{0.000} & \good{ 0.017} \\            
    \ftblmidrules{}
    \color{black}{0.793} &                    \color{black}{0.803} &       \good{0.010} & - & - 
     \end{tabular}
   
%     %%%%%%%%%%%%%%%%%%%%%%%%%%%%%%%%%%%%%%%%%
    \begin{tabular}{Hcccc}
    \multicolumn{3}{c}{\textsc{Test ECE}}  & \multicolumn{2}{c}{\textsc{Interventions}} \\
    \ftblmidrules{}
    $\truerisk{}{\clf{}}$ &  $ \truerisk{\gb}{\plf{\gb}}$ &  $\Delta_{\gb}$ & 
    \cell{r}{Assign $\clf{}$} & 
    \cell{r}{Assign $h^\textrm{dcp}_{\gb}$}\\
    %$\gb$ & %$n$ &  %$h_0(\xb) &  %$h(\xb, \gb) &  %$\Remp{\gb}h(\xb, \gb)) - \Remp{\gb}(h_0(\xb))$\\
    \ftblmidrules{}
    34.3\% &                    32.2\% &        \good{2.1\%} & 2.1\% & \good{11.9\%} \\   
    12.1\% &                    10.1\% &        \good{2.0\%} & 2.0\% & \good{0.03\%}\\    
    16.5\% &                    14.7\% &        \good{1.8\%} & 1.8\% & \good{5.3\%} \\        
    17.9\% &                    18.1\% &       \violation{-0.0\%} &  \closecall{0.0}\% & \good{6.4\%} \\            
    9.2\% &                    10.6\% &       \violation{-1.4}\% & \closecall{0.0}\% & 1.4\%\\            
    13.7\% &                    13.5\% &        \good{0.0\%} & 0.0 & 0.0\% \\            
    \ftblmidrules{}
    \color{black}{4.9\%} &                    \color{black}{4.7\%} &       \good{0.2\%} & - & - %\\
    % \ftblmidrules{}
    \end{tabular}
    }
    \caption{Fair use evaluation of a personalized logistic regression model with a one-hot encoding of group attributes. As shown, personalization can improve overall performance while reducing performance for specific groups ({\color{fbad}{red}}). This result holds across all performance metrics. 
    In such cases, we can resolve fair use violations and improve the gains from personalization by assigning personalized predictions to each group with multiple models. By this, we mean selecting from one of three available models which provides the most accurate predictions for a group: a generic model $\clf{}$, the personalized model $\plf{\gb}$, or a decoupled model $h^\textrm{dcp}$. We highlight cases where assigning predictions from one of these models led to a gain in green, and where it resolved a violation in yellow.}
    \label{Table::KidneyResults}
\end{table*}

\section{Mortality in Acute Kidney Injury}
\label{Sec::Demonstration}

In this section, we audit fair use for a mortality prediction model for patients with acute kidney injury. Our results demonstrate how evaluating the gains of personalization can inform model development and improve simple interventions to mitigate harm.

\subsection{Setup}
\label{Sec::DemonstrationSetup}

We consider a mortality prediction task for critically-ill patients who receive continuous renal replacement therapy. The data contains $n = 2,066$ patients from MIMIC III and IV~\citep[][]{johnson2016mimic} and includes $d = 78$ features related to their health, lab tests, length of stay, and potential for organ failure. Here, $y_i = +1$ if patient $i$ dies in the ICU and $\textrm{Pr}(y_i=+1)=51.1\%$. We train personalized models using the setup in \cref{Sec::ExperimentalSetup}, and evaluate fair use for groups defined by the attributes $\textgp{sex} \in \{\textgp{male}, \textgp{female}\}$ and $\textgp{race} \in \{\textgp{white}, \textfn{black}, \textgp{other}\}$. 

% mortality prediction helps clinicians make more effective and efficient treatment decisions in the ICU. 
\subsection{Results}
\label{Sec::DemonstrationResults}

We show performance for the personalized logistic regression model with a one-hot encoding in \cref{Table::KidneyResults}, and include results for other configurations in \cref{Appendix::Experiments}. Our findings show that personalization yields uneven gains at a group level, producing fair use violations across prediction tasks and model classes. In this case, the gains in error across group range from -5.2\% to 6.8\%, and two groups experience statistically significant fair use violations: (\textgp{male, black}) and (\textgp{male, other}).

\paragraph{On the Use of Race} 

Clinical prediction models include group attributes whenever there is a plausible biological relationship between group membership and the outcome of interest or social determinants of health. These norms have led to the development of models that use race and ethnicity~\cite{eneanya2019reconsidering,vyas2020hidden, flamm1997vaginal, haukoos2012derivation, levey2009new, moore2014derivation, gail1989projecting, kanis2007glucocorticoid}. Recently, \citet{vyas2020hidden} discuss how such models can inflict harm and urge physicians to check if ``race correction is based on robust [statistical] evidence." Our results highlight how a fair use audit can yield evidence that serves to guide the use ``race correction" in such cases. Here, checking rationality shows that a race-specific model can reduce performance for specific groups -- e.g., (\textgp{male, black}) and (\textgp{male, other}). Checking envy-freeness reveals that groups expect better performance by misreporting group membership -- e.g., (\textgp{male,other}) would experience a 5.6\% gain in test error by reporting any other race. 

%\paragraph{On Race as a Proxy and Proxies of Race}

In tasks where race improves performance, race may act as a proxy for broader social determinants of health. Thus, a model that includes race may act as a ``smoke screen" in that it attributes differences in health outcomes to an immutable factor, and perpetuates inaction on the root causes of health disparities~\citep{perez2017now}. 
% %
%In the same way that height, weight, and bone mass are proxies for sex~\cite{nieves2005males}, many subtle statistical differences between patient subgroups to infer patient demographics at super-human performance. This could be a factor in the lack of model improvement once direct demographic variables are included. 
% %
Given these uncertainties, we advocate that race should only be included in a clinical prediction model when there is evidence of gain. Regardless of its use in prediction, collecting information about race and ethnicity is necessary to measure model performance across these groups. In such cases, one should be careful to disclose the purposes of data collection -- stating that it will be used to evaluate performance but not to assign personalized predictions. 
In tasks where race does not improve model performance, models may exhibit differences in performance across racial groups -- as data may encode proxies of race in redacted notes~\citep{adam2022write}, or even band-pass filtered images~\citep{gichoya2022ai}. %In the same way that height, weight, and bone mass are proxies for sex~\cite{nieves2005males}, models can use other features to account for race leading to a lack of gains across groups.

\paragraph{Interventions}

We build on our results to discuss interventions that can resolve fair use violations and broaden the gains to personalization by using multiple models. These are simple interventions that have the benefit of being broadly applicable -- i.e., we can use them to mitigate harm from fair use violations for any prediction task and model class.

\emph{Assigning a Generic Model}.\quad We assign groups who experience a fair use violation the predictions from a generic model $\clf{}$. This intervention will resolve all fair use violations in a way that strictly improves performance. In this case, it resolves all rationality violations (2/3/2 in terms of error/AUC/ECE respectively). We also observe a potential to reduce data usage in deployment: seeing how both (\textgp{male, black}) and (\textfn{male, other}) experience a fair use violation in terms of error, we could soliciting race for all \textfn{male} patients and reduce test error by 1\% (as the loss in accuracy for (\textfn{white, male}) are offset by the gain in accuracy for (\textfn{male,black}) and (\textfn{male, other}).

\emph{Assigning a Decoupled Model}.\quad  We assign groups who experience a fair use violation predictions from the best of a generic model, personalized model, or a \emph{decoupled model} $h^\textrm{dcp}_{\gb{}}$-- i.e., a model trained using only data from their group. 
While this approach may not resolve fair use violations, it can produce surprisingly large gains as decoupling effectively personalizes the entire model development pipeline. Our results in \cref{Table::KidneyResults} show the potential gains of this intervention across all performance metrics. 
Focusing on error, we see that one can: (1) eliminate fair use violations for (\textgp{male,black}) and (\textgp{male,other}); (2) greatly improve accuracy for (\textfn{female,black}) who experience a gain of {\bf 37.3\%} from a previous accuracy of less than 50\%; and (3) improve overall gains by $6.2\%$.
We observe similar effects across other configurations and model classes. 
\section{Concluding Remarks}
\label{Sec::ConcludingRemarks}

Machine learning models that are personalized with group attributes can fail to improve performance for all groups who provide personal data. Our results underscore the need to evaluate fair use when developing models with group attributes that are protected, sensitive, self-reported, or costly to acquire ~\citep[e.g.,][]{steyerberg2010assessing,wallace2011framework,collins2015transparent,cowley2019methodological,bouwmeester2012reporting,cabitza2021need,kent2020predictive}. Evaluating fair use is a routine procedure that whose results can be summarized and communicated in a model report~\citep[][]{mitchell2019model,arnold2019factsheets,barocas2021designing,bynum2021disaggregated,cabrera2019fairvis} -- and that can be used to flag instances where personalization reduces performance for specific groups and guide interventions that broaden the gains of personalization.

\paragraph{Limitations} 

Our work assumes that a gain in performance is a suitable ``stand-in" for preference or harm, which holds in tasks where every group benefits from a more accurate model. This assumption may not hold when, for example, models are trained to use proxy labels, or groups may prefer a specific prediction over the most accurate prediction. 

In closing, we caution that fair use should be considered a safeguard against ``worsenalization" rather than a rubber stamp for consent. In effect, fair use is not an individual-level guarantee. The gains associated with fair use conditions reflect average measures of performance over individuals in a group. In tasks where these gains are reported to individuals, they should be presented alongside information that summarizes the impact of personalization on their prediction -- e.g., the degree of change in individual predictions due to personalization, and the degree of representation in the sample used to evaluate the gains of personalization.

\section*{Acknowledgements} 
We thank the following individuals for helpful discussions: Flavio Calmon, Katherine Heller, Sanmi Koyejo, Ziad Obermeyer, Charlie Marx, Stephen Pfohl, Emma Pierson, Kush Varshney, and Haoran Zhang. This work was supported by funding from the National Science Foundation IIS 2040880, the NIH Bridge2AI Center Grant U54HG012510, and the Wellcome Trust. 

{%
\small
\clearpage
\bibliography{fair_use_hc}
}

\clearpage
\appendix
\onecolumn

\section{Notation}
\label{Table::TableOfNotation}

We provide a list of the notation used throughout the paper in~\cref{Table::Notation}.

\begin{table}[h!]
\centering
\resizebox{0.8\textwidth}{!}{
\begin{tabular}{l>{\quad}l}
\textheader{Symbol} & \textheader{Meaning}\\ 
\toprule
$\xb_i = (x_{i,1}, x_{i,2},\ldots, x_{i,d})$ & feature vector of example $i$ \\ 
$y_i \in \Y$ & label of example $i$ \\ 
$\gb_i \in \{g_{i,1}, g_{i,2}, \ldots, g_{i,k}\}$ & group membership of example $i$\\
$\G = \G_1 \times \G_2 \times \ldots \times \G_k$ & space of group attributes\\ 
$m = |\G|$ & number of intersectional groups \\[1em]
$\n{\gb} := \sum 1[\gb_i = \gb]$ & number of examples of group $\gb\in\G$ \\
$\nplus{\gb} := \sum 1[\gb_i = \gb{},\, y_i = +1]$ &  number of examples of group $\gb\in\G$ with $y_i = +1$ \\
$\nminus{\gb} :=  \sum 1[\gb_i = \gb{},\, y_i = -1]$ & number of examples of group $\gb\in\G$ with $y_i = -1$ \\[1em]

$\plf{}: \X \times \G \to \Y$ & personalized model\\ 
$\Hset$ & hypothesis class of personalized models\\ 
$\plf{\gb}: \X \times \G \to \Y$ & personalized classifier where group membership is reported truthfully as $\gb{}$\\ 
$\clf{}: \X \to \Y$ & generic model \\ 
$\Hset_0$ & hypothesis class of generic models \\ 
$\truerisk{\gb}{\plf{\gb'}}$& true risk of model $\clf{}$ of group $\gb$ if they report $\gb'$\\
$\emprisk{\gb}{\plf{\gb'}}$& empirical risk of model $\plf{}$ of group $\gb$ if they report $\gb'$\\[1em]
$\truegain{\gb}{h}{h'}$ & gain (i.e., reduction in true risk) for group $\gb$ when using $h$ instead of $h'$\\
$\truegain{\gb}{\plf{\gb}}{\clf{}}$ & rationality gap for group $\gb$ under model $\plf{}$ \\
$\truegain{\gb}{\plf{\gb}}{\plf{\gb'}}$ & envyfreeness gap for group $\gb$ under model $\plf{}$\\
\bottomrule
\end{tabular}
}
\caption{Notation}
\label{Table::Notation}
\end{table}

\clearpage
\section{Supporting Material for \cref{Sec::FailureModes}}
\label{Appendix::Theory}

\subsection{Additional Failure Modes of Personalization}
\label{Appendix::FailureModes}

We describe additional mechanisms that lead personalized models to exhibit fair use violations. The mechanisms below reflect failure modes that arise in later stages of the machine learning pipeline, and that are more difficult to address through interventions.  

\paragraph{ERM with a Surrogate Loss Function}

Consider a setting where we want a personalized model that maximizes classification accuracy -- i.e., one that minimizes the 0--1 loss. If we fit this classifier using a linear SVM – e.g., by solving an ERM problem that optimizes the hinge loss – the approximation error between the 0-1 loss and the hinge loss can produce a fair use violation (see Figure~\ref{Fig::FailureModeSurrogateLoss}). This example is specifically designed to avoid fair use violations that stem from model misspecification.

\newcommand{\addfmplot}[1]{\includegraphics[trim=0.0in 0.0in 0.0in 0in, clip, width=0.4\textwidth]{figure/failure_modes/fm5/#1}}

\begin{figure}[!h]
    \centering
    \resizebox{\linewidth}{!}{
    \begin{tabular}{cc}
    \cell{c}{\textsf{Group $A$}} & \cell{c}{\textsf{Group $B$}} \\ 
    \addfmplot{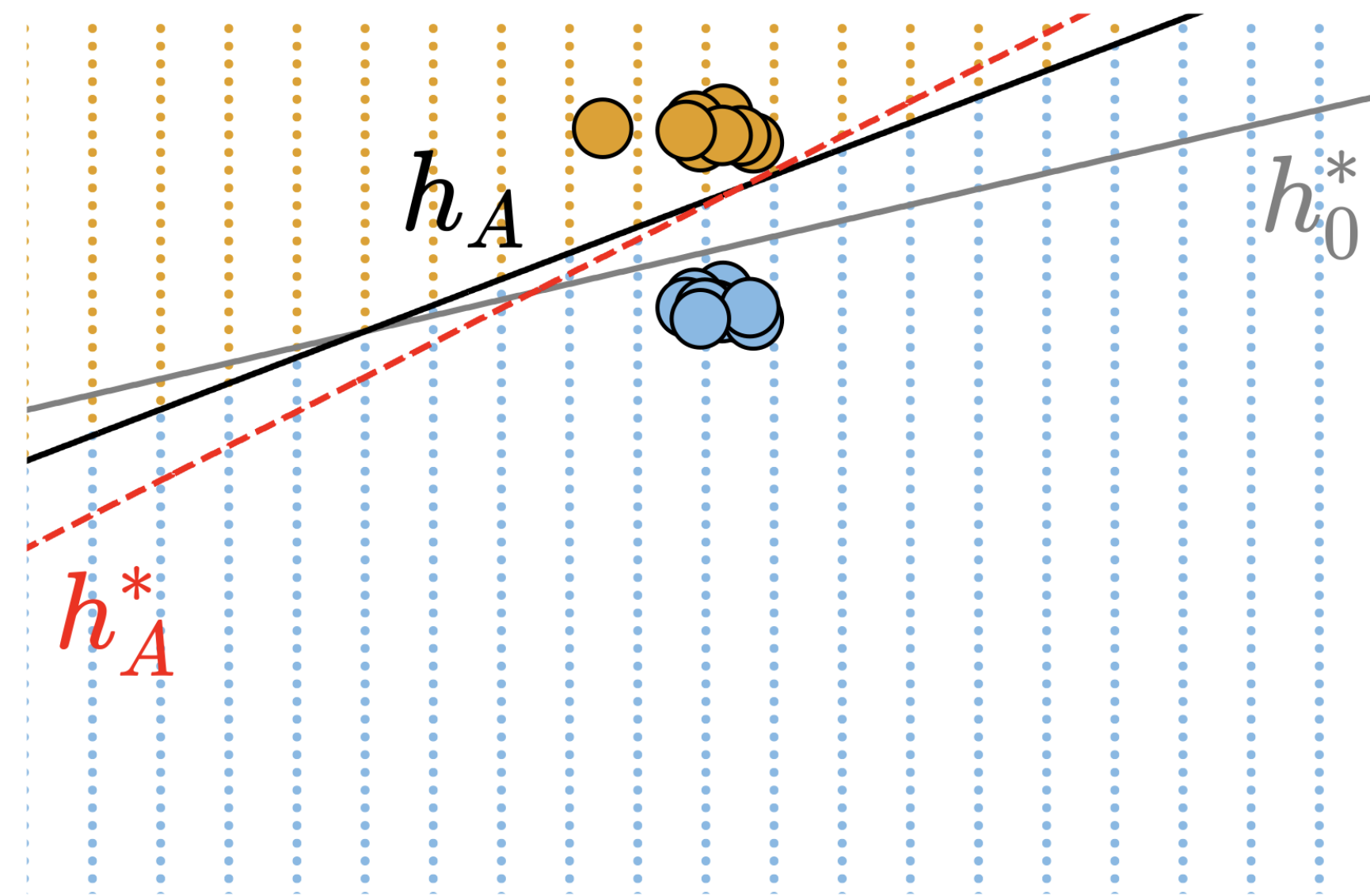} &
    \addfmplot{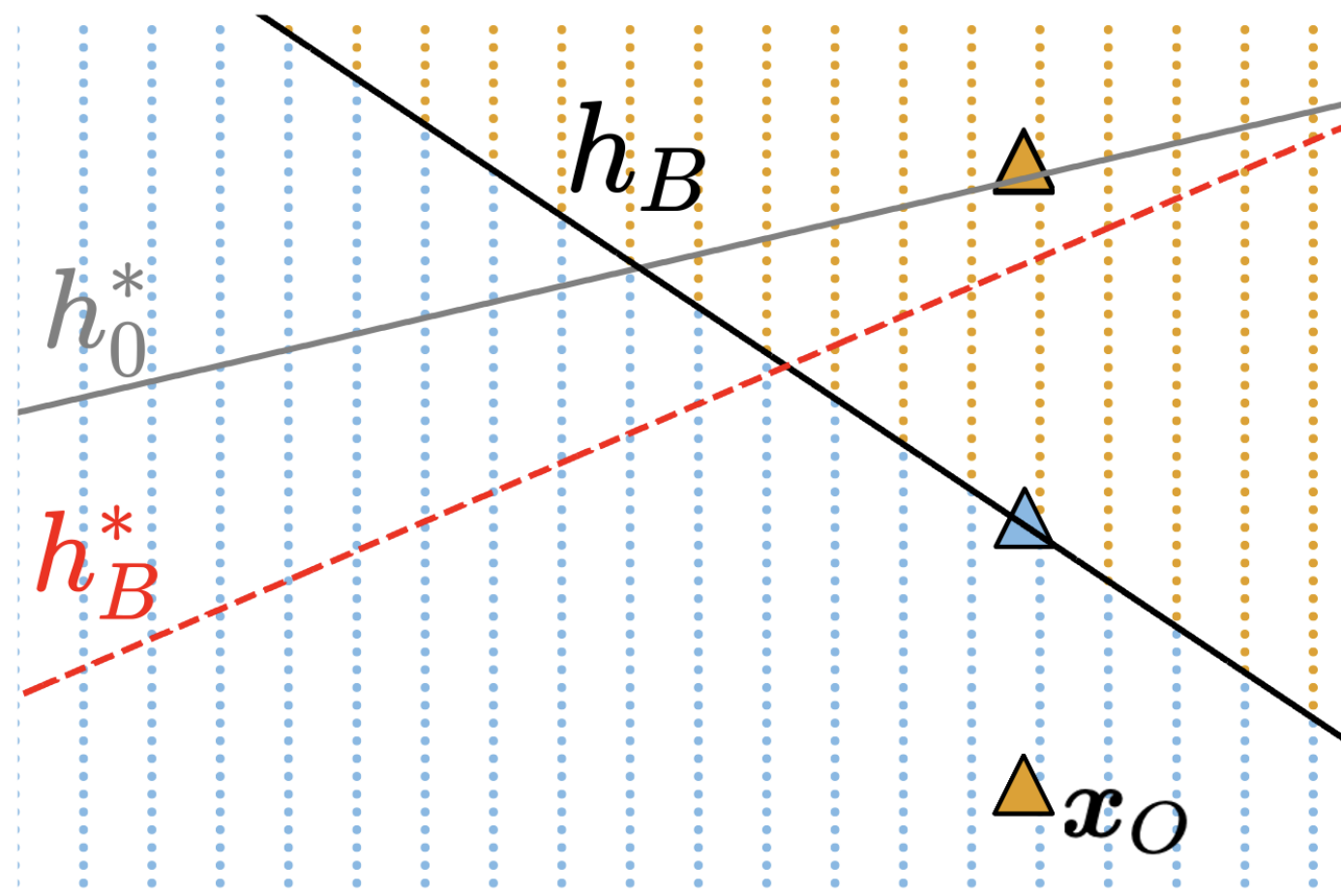}
    \end{tabular}
    }
    \caption{Fair use violations resulting from empirical risk minimization with a surrogate loss function. We consider a classification task with two features $\xb = (x_1, x_2)$ and one group attribute $\gb \in \{A, B\}$ in which we fit a linear SVM $\plf{\gb}$ but evaluate the the gains of personalization in terms of the error rate (i.e., hinge loss vs. 0-1 loss). We plot the data for group $A$ and group $B$ separately, and show the generic classifier ($\clf{}$; grey) and the personalized classifiers for the corresponding group ($h_A$ or $h_B$; black). In this case, the personalized model produces a fair use violation for Group $B$ due to an outlier $\xb_O$. As a baseline for comparison, we show the personalized models that we would obtain by optimizing an exact loss function (i.e., 0-1 loss, which matches the performance metric that we use to evaluate the gains for personalization). As shown, we would expect to avoid this violation had we fit a model by optimizing the 0--1 loss directly.}
    \label{Fig::FailureModeSurrogateLoss}
\end{figure}

\FloatBarrier
\paragraph{Generalization \& Dataset Shifts} 

Fair use violations can arise in deployment. Small samples may distort the relative prevalence of each group, leading ERM to return a personalized model or suboptimal generic model. In \cref{Fig::FailureModeSamplingError}, we show how fair use violations occur when sampling bias results in a difference in the training data distribution and the true distribution. Here, we sample data from the true distribution where the small sample size or sampling bias results in a label shift for one specific group. Likewise, violations can  arise as a result of changes in the data distribution~\citep[i.e., dataset shift ][]{quinonero2008dataset,finlayson2021clinician,guo2021systematic} (see \cref{Fig::FailureModeLabelShift})

\begin{figure}[!htb]
    \centering
\renewcommand{\ftblmidrules}[0]{%
        \cmidrule(lr){1-2}
        \cmidrule(lr){3-4}
        \cmidrule(lr){5-6}
        \cmidrule(lr){7-8}
        \cmidrule(lr){9-11}
        \cmidrule(lr){12-14}
    }
    \resizebox{\linewidth}{!}{\begin{tabular}{rrrrrrrrrrrrrr}
        \ftblheader{2}{c}{Group} &
        \ftblheader{2}{c}{Training Data} &
        \ftblheader{2}{c}{Data Distribution} &
        \ftblheader{2}{c}{Predictions}  &
        \ftblheader{3}{c}{Observed Performance} &  
        \ftblheader{3}{c}{True Performance} \\
        % %
        \ftblmidrules{}
        
        $g_1$ & $g_2$ & 
         $\nplus{}$ & $\nminus{}$  &
         $\nplus{}$ & $\nminus{}$ &
         $\clf{}$ & $\plf{\gb}$ &
         $\truerisk{\gb}{\clf{}}$ & $\truerisk{\gb}{\plf{}}$ &
         $\truegain{\gb}{\plf{\gb}}{\clf{}}$ & 
         $\truerisk{\gb}{\clf{}}$ & $\truerisk{\gb}{\plf{}}$ &
         $\truegain{\gb}{\plf{\gb}}{\clf{}}$ \\
        \ftblmidrules{}
        % %
         $0$ & $0$& 
         $65$ & $60$ &
         $130$ & $120$ &
         $+$ & $+$ &
         $60$ & $60$ & 
         \cellcolor{funk!60} $0$ &
         $120$ & $120$ & 
         \cellcolor{funk!60} $0$ \\
         $1$ & $0$& 
         $60$ & $65$ &
         $120$ & $130$ &
         $+$ & $-$ &
         $65$ & $60$ & 
         \cellcolor{fgood!60} $5$ &
         $130$ & $120$ & 
         \cellcolor{fgood!60} $10$ \\
         $0$ & $1$ &
         $60$ & $65$ &
         $130$ & $120$ &
         $+$ & $-$ &
         $65$ & $60$ & 
         \cellcolor{fgood!60} $5$ &
         $120$ & $130$ & 
         \cellcolor{fbad!60} $-10$ \\
         $1$ & $1$ &
         $70$ & $55$ &
         $140$ & $110$ &
         $+$ & $+$ &
         $55$ & $55$ & 
         \cellcolor{funk!60} $0$ &
         $110$ & $110$ & 
         \cellcolor{funk!60} $0$ \\
        \ftblmidrules{}
         & 
         \textbf{Total} &
         $255$ & $245$ &
         $520$ &  $480$ & 
          &   & 
         $245$ & $235$ &
         \cellcolor{fgood!60} $10$ & 
         $480$ & $470$ & 
         \cellcolor{funk!60} $0$ \\
    \end{tabular}
}
\caption{Fair use violations can arise when personalizing models on small samples. Here, we show a 2D classification task in which a personalized model only exhibits fair use violations in deployment. Here, group  $(1,0)$ experiences an gain once the model is deployment. In contrast, group $(0,1)$ experiences a fair use violation as a result of sampling error.}
    \label{Fig::FailureModeSamplingError}
\end{figure}

\begin{figure}[!htb]
    \centering

    \renewcommand{\ftblmidrules}[0]{%
        \cmidrule(lr){1-2}
        \cmidrule(lr){3-4}
        \cmidrule(lr){5-6}
        \cmidrule(lr){7-8}
        \cmidrule(lr){9-11}
        \cmidrule(lr){12-14}
    }
    \resizebox{\linewidth}{!}{\begin{tabular}{rrrrrrrrrrrrrr}
        \ftblheader{2}{c}{Group} &
        \ftblheader{2}{c}{Training Data} &
        \ftblheader{2}{c}{Data Distribution} &
        \ftblheader{2}{c}{Predictions}  &
        \ftblheader{3}{c}{Observed Performance} &  
        \ftblheader{3}{c}{True Performance} \\
        % %
        \ftblmidrules{}
        
        $g_1$ & $g_2$ & 
         $\nplus{}$ & $\nminus{}$  &
         $\nplus{}$ & $\nminus{}$ &
         $\clf{}$ & $\plf{}$ &
         $\truerisk{\gb}{\clf{}}$ & $\truerisk{\gb}{\plf{}}$ &
         $\truegain{\gb}{\plf{\gb}}{\clf{}}$ & 
         $\truerisk{\gb}{\clf{}}$ & $\truerisk{\gb}{\plf{}}$ &
         $\truegain{\gb}{\plf{\gb}}{\clf{}}$ \\
        \ftblmidrules{}
        % %
         $0$ & $0$
         & $20$ & $0$ &
         $20$ & $0$ &
         $+$ & $+$ &
         $0$ & $0$ & 
         \cellcolor{funk!60} $0$ &
         $0$ & $0$ & 
         \cellcolor{funk!60} $0$ \\
         $1$ & $0$
         & $5$ & $25$ &
         $5$ & $25$ &
         $+$ & $-$ &
         $25$ & $5$ & 
         \cellcolor{fgood!60} $20$ &
         $25$ & $5$ & 
         \cellcolor{fgood!60} $20$ \\
         $0$ & $1$
         & $5$ & $25$ &
         \cellcolor{funk!60} $30$ & $25$ &
         $+$ & $-$ &
         $25$ & $5$ & 
         \cellcolor{fgood!60} $20$ &
         $20$ & $30$ & 
         \cellcolor{fbad!60} $-10$ \\
         $1$ & $1$
         & $20$ & $0$ &
         $20$ & $0$ &
         $+$ & $+$ &
         $0$ & $0$ & 
         \cellcolor{funk!60} $0$ &
         $0$ & $0$ & 
         \cellcolor{funk!60} $0$ \\
        \ftblmidrules{}
         & 
         \textbf{Total}  &
         $50$ &  $50$ & 
         $75$ & $45$ &
         &   & 
         $50$ & $10$ &
         \cellcolor{fgood!60} $40$ & 
         $45$ & $35$ & 
         \cellcolor{fgood!60} $10$
    \end{tabular}
    }
    \caption{Label shift produces a fair use violation. Here, we train a linear classifier on a dataset with [one binary feature and one binary group attribute]. As shown, personalization leads to overall improvement reducing aggregate reduce from 50 to 24 and group-specific improvement on the training data. However, not all groups perform equally well in deployment. While groups $(0,1)$ and $(1,1)$ see improvements, a violation (red) occurs for group $(1,0)$ due to the label shift where positive examples in the true distribution for group $(0,1)$ (highlighted in yellow) are undersampled in the training data.}
    \label{Fig::FailureModeLabelShift}
\end{figure}

\subsection{Missing Proofs}
\label{Appendix:TheoreticalGuarantees}

We provide the proofs for our sufficient conditions described in~\cref{Sec::FailureModes}. We start with a simple condition to ensure the empirical risk minimizer over $\Hset$ can return a model that assigns the same predictions as a generic model for every group.
\begin{definition}
\label{Def:Extends}
A personalized model class $\Hset$ \emph{extends} a generic model class $\Hset_0$ if for every personalized model $h \in \Hset$, there exists a generic model $\clf{} \in \Hset_0$ such that $\clf{}(\xb) = \plf{}(\xb, \gb)$ for all $\xb \in \X$ and all groups $\gb \in \G$.
\end{definition}
This is a basic condition that is often satisfied in practice, and can be guaranteed by practitioners during model specification. Intuitively the condition is meant to rule out instances where a personalized model exhibits a rationality violation because it is required to account for group membership (see e.g., \cref{Fig::FailureModeFeatureSelection}).

\paragraph{\cref{Prop::TrainingGuarantee}}

Consider training a personalized model by ERM $\plf{} \in \argmin_{h \in \Hset}\emprisk{}{h}$, and evaluating its gains to personalization with respect to a generic model $\clf{} \in \argmin_{h \in \Hset_0}\emprisk{}{h}$ where $\Hset_0 \subseteq \Hset.$ The personalized model $\plf{}$ obeys fair use in terms of empirical risk so long as: $$\emprisk{\gb}{\plf{}} = \emprisk{\gb}{\plf{\gb}} ~\textrm{for all groups}~\gb \in \G.$$

\begin{proof}
Say that we have a personalized model $\plf{} \in \argmin_{h \in \Hset}\emprisk{}{h}$ that obeys $\emprisk{\gb}{\plf{}} = \emprisk{\gb}{\plf{\gb}}$ for all groups $\gb \in \G$. This implies that $\emprisk{\gb}{\plf{\gb}} \leq \emprisk{\gb}{\plf{}}$ for any model $\plf{} \in \Hset$ and any group $\gb \in \G$. 
Since $\clf{} \in \Hset$, we have that $\emprisk{\gb}{\plf{\gb}} \leq \emprisk{\gb}{\clf}$ for all groups $\gb \in \G$. Thus, rationality holds for all groups $\gb \in \G$.
Likewise, since $\plf{\gb^{'}} \in \Hset$, we have that $\emprisk{\gb}{\plf{\gb}} \leq \emprisk{\gb}{\plf{\gb^{'}}}$ for all groups $\gb, \gb' \in \G$. Thus, envy-freeness holds for all groups $\gb \in \G$.

\end{proof}

\paragraph{\cref{Prop::TestingGuarantee}} Consider a personalized model $\plf{}: \X \times \G \to \Y$ that ensures rationality and envy-freeness for group $\gb$ in terms of empirical risk. Denote the empirical gains in rationality and envy-freeness for group $\gb$ as:
\begin{align*}
\empgapmin{\gb} := \empgap{\gb}{\plf{\gb}}{\clf{}}, \qquad \hat{\gamma}_{\gb} := \min_{\gb' \in \G / \{\gb\}} \empgap{\gb}{\plf{\gb}}{\plf{\gb'}}
\end{align*}
If $\empgapmin{\gb} > 0$, then rationality for group $\gb$ generalizes with probability at least $1 - \delta$ as long as:
\begin{align}
\n{\gb} \geq \frac{4\VC{\Hset} \log \left(\frac{2n_{\gb}}{\VC{\Hset}} + 1\right) + \log \left(\tfrac{8}{\delta}\right)}{\empgapmin{\gb}^2} \label{Eq::RatBound} 
\end{align}
If $\hat{\gamma}_{\gb} > 0$, then envy-freeness for group $\gb$ generalizes with probability at least $1 - \delta$ as long as:
\begin{align}
 \n{\gb} &\geq \frac{4\VC{\Hset} \log \left(\frac{2n_{\gb}}{\VC{\Hset}} + 1\right) + \log \left(\tfrac{8m}{\delta}\right)}{\hat{\gamma}_{\gb}^2}
 \label{Eq::EFBound} 
\end{align}

The proof of \cref{Prop::TestingGuarantee} are based on a generalized version of a lemma from \citet{ustun2019fairness} which assumes that the dimension of the hypothesis class is finite whereas we use VC-dimension instead of the dimension of the hypothesis class.

\begin{lemma}[Generalization of Gains]
\label{Lem::GeneralizationOfGain}

Consider a pair of classifiers $\plf{a}$ and $\plf{b}$ from a hypothesis class $\Hset$ with VC-dimension $\VC{\Hset}$. If the empirical risk of each classifier on group $\gb$ satisfy $\empgain{\gb}{\plf{a}}{\plf{b}} := \emprisk{\gb}{\plf{b}} - \emprisk{\gb}{\plf{a}} > 0$, then for any $\delta > 0$, the corresponding gap in true risk will satisfy $\truegain{\gb}{\plf{a}}{\plf{b}} > 0$ with probability at least $1 - \delta$ so long as:
\begin{align}
    \sqrt{\frac{4\VC{\Hset} \left(\log{\frac{2\n{\gb}}{\VC{\Hset}}} + 1 \right) + \log{(\frac{8}{\delta})}}{\n{\gb}}} \leq \empgain{\gb}{\plf{a}}{\plf{b}}\label{Eq::GenGainBound}.
\end{align}
\begin{proof}

The proof applies a standard concentration inequality~\cite{mitchell2007machine} to bound the generalization error of a classifier over groups as follows.
Given a classifier $\plf{} \in \Hset{}$ from hypothesis class $\Hset$ with VC-dimension $\VC{\Hset}$, and any $\delta > 0$, the generalization error of $\plf{}$ on group $\gb \in \G$ with $\n{\gb}$ will obey the following inequality~\cite{mitchell2007machine} with probability at least $1 - \tfrac{\delta}{2}$:
\begin{align}
\left| \emprisk{\gb}{\plf{}} - \truerisk{\gb}{\plf{}} \right| \leq \sqrt{\frac{\VC{\Hset} \left(\log{\frac{2\n{\gb}}{\VC{\Hset}} + 1} \right) + \log{\frac{8}{\delta}}}{\n{\gb}}}. \label{Eq::GenBoundGroup}
 \end{align}
We denote the quantity on the right hand side of Eq. \eqref{Eq::GenBoundGroup} as the bounding function $B(\n{\gb},\Hset,\delta) := \sqrt{\frac{\VC{\Hset} \left(\log{\frac{2\n{\gb}}{\VC{\Hset}} + 1} \right) + \log{\frac{8}{\delta}}}{\n{\gb}}}$. Given the bounding function $B(\n{\gb},\Hset,\delta)$, \cref{Lem::GeneralizationOfGain} states that for any $\delta > 0$, with probability at least $1 - \delta$,
\begin{align*}
    2B(\n{\gb},\Hset,\delta) \leq \empgap{\gb}{\plf{a}}{\plf{b}} \quad \implies \quad \truerisk{\gb}{\plf{b}} - \truerisk{\gb}{\plf{a}} \geq 0
\end{align*}
We will prove the statement by showing that the condition on the left hand side implies the condition on the right hand side. Assume that the condition on the left hand side holds so that $2B(\n{\gb},\Hset,\delta) \leq \empgap{\gb}{\plf{a}}{\plf{b}}$. Then we can observe that the right hand side is bounded as follows:
\begin{align*}
\truerisk{\gb}{\plf{b}} - \truerisk{\gb}{\plf{a}}
&= \truerisk{\gb}{\plf{b}} - \truerisk{\gb}{\plf{a}} + \emprisk{\gb}{\plf{a}} - \emprisk{\gb}{\plf{a}} + \emprisk{\gb}{\plf{b}} - \emprisk{\gb}{\plf{b}} \\
&= 
\underbrace{\truerisk{\gb}{\plf{b}} - \emprisk{\gb}{\plf{b}}}_{\geq -B(\n{\gb}, \Hset, \delta)} 
+ \underbrace{\emprisk{\gb}{\plf{a}} - \truerisk{\gb}{\plf{a}}}_{\geq -B(\n{\gb}, \Hset, \delta)} 
+ \underbrace{\emprisk{\gb}{\plf{b}} - \emprisk{\gb}{\plf{a}}}_{:=\empgap{\gb}{\plf{a}}{\plf{b}}} \\
& \geq  -2B(\n{\gb}, \Hset, \delta) + \empgap{\gb}{\plf{a}}{\plf{b}} \\
& \geq 0
\end{align*}
Thus we have that $\truerisk{\gb}{\plf{b}} - \truerisk{\gb}{\plf{a}} \geq 0$ whenever $2B(\n{\gb},\Hset,\delta) \leq \empgap{\gb}{\plf{a}}{\plf{b}}$. This completes the proof.
\end{proof}
\end{lemma}

We now present the proof to  \cref{Prop::TestingGuarantee}.
\begin{proof}
We recover the bounds by applying \cref{Lem::GeneralizationOfGain}.
We start with the bound on rationality in \cref{Eq::RatBound}. Given that $\empgapmin{\gb} > 0$, we apply \cref{Lem::GeneralizationOfGain} to the personalized and model $\plf{\gb}$ and the generic model $\clf$ to obtain:
\begin{align*}
    \sqrt{\frac{4\VC{\Hset} \left(\log{\frac{2\n{\gb}}{\VC{\Hset}}} + 1 \right) + \log{(\frac{8}{\delta})}}{\n{\gb}}} \leq \
    \empgapmin{\gb} \\
     \n{\gb} \geq \frac{4\VC{\Hset} \left(\log{\frac{2\n{\gb}}{\VC{\Hset}}} + 1 \right) + \log{(\frac{8}{\delta})}}{{\empgapmin{\gb}}^2}
\end{align*}
We now consider the bound on envy-freeness \cref{Eq::EFBound}. Given that $\hat{\gamma}_{\gb} > 0$, we apply \cref{Lem::GeneralizationOfGain} to the personalized model $\plf{\gb}$ and $\plf{\gb'}$ for all $\gb, \gb' \in \G$. This produces $m - 1$ preferences to generalize. Given that $m - 1 \leq m$, we apply~\cref{Lem::GeneralizationOfGain} with probability $1 - \frac{\delta}{m}$. Doing so and inverting for $n_{\gb}$ proves the result.
\begin{align*}
    \sqrt{\frac{4\VC{\Hset} \left(\log{\frac{2\n{\gb}}{\VC{\Hset}}} + 1 \right) + \log{\frac{8m}{\delta}}}{\n{\gb}}} \leq \
    \hat{\gamma}_{\gb} \\
     \n{\gb} \geq \frac{4\VC{\Hset} \left(\log{\frac{2\n{\gb}}{\VC{\Hset}}} + 1 \right) + \log{(\frac{8m}{\delta})}}{{\hat{\gamma}_{\gb}}^2}
\end{align*}
\end{proof}

\section{Additional Information on Datasets}
\label{Appendix::Datasets}

In this Appendix, we include additional information on the datasets used in \cref{Sec::Experiments} and \cref{Sec::Demonstration}. We present a summary of the goals and characteristics for each dataset in \cref{Table::DatasetOverview}. We include a brief description of each dataset and preprocessing steps taken below.

\begin{table}[h]
    \centering
    \resizebox{\linewidth}{!}{
    \begin{tabular}{lrrlll}
    \textheader{Dataset} & $n$ & $d$ & \textheader{Group Attributes} -- $\G$ & \textheader{Prediction Task} & \textheader{Reference} \\
    \toprule

    \textds{apnea} & 
    $1,152$ & $26$ & 
    \cell{l}{$\textfn{Age} \times \textfn{Sex} = \{$<30$, 30~\textfn{to}~60, 60+\} \times \{\textfn{Male}, \textfn{Female}\}$} &
    patient has obstructive sleep apnea & 
    \citet{ustun2016clinical} \\ 
    \midrule

    \textds{cardio\_eicu} & 
    $1,341$ & $49$ &
    \cell{l}{$\textfn{Age} \times \textfn{Sex} =\{\textfn{Young}, \textfn{Old}\} \times  \{\textfn{Male}, \textfn{Female}\}$} &
    patient with cardiogenic shock dies & \citet{pollard2018eicu} \\ 
    \midrule
    
    \textds{cardio\_mimic} & 
    $5,289$ & $49$ & 
    \cell{l}{$\textfn{Age} \times \textfn{Sex} = \{\textfn{Young}, \textfn{Old}\}\times \{\textfn{Male}, \textfn{Female}\}$} &
    patient with cardiogenic shock dies & \citet{johnson2016mimic} \\ 
    \midrule
    
    \textds{heart} & 
    $181$ & $26$ & 
    \cell{l}{$\textfn{Age} \times \textfn{Sex} = \{\textfn{Young}, \textfn{Old}\} \times  \{\textfn{Male}, \textfn{Female}\}$} &
    patient has heart disease & 
    \citet{detrano1989international} \\ 
    \midrule
    
    \textds{kidney} & 
    $2,066$ & $78$ &
    \cell{l}{$\textfn{Sex} \times \textfn{Race} = \{\textfn{Male}, \textfn{Female}\} \times  \{\textfn{White, Black, Other}\}$} & 
    mortality of patient on CRRT &
    \citet{zhang2021learning} \\
    \midrule
    
    \textds{mortality} & 
    $21,139$ & $484$ &
     \cell{l}{$\textfn{Age} \times \textfn{Sex} = \{<30, 30~\textfn{to}~60, 60+\} \times  \{\textfn{Male}, \textfn{Female}\}$} & 
    mortality of patient in ICU &
    \citet{harutyunyan2019multitask} \\
    \midrule
    
    \textds{saps} & 
    $7,797$ & $36$ & 
     \cell{l}{$\textfn{Age} \times \textfn{HIV} = \{\leq 30, 30+\} \times  \{\textfn{Positive}, \textfn{Negative}\}$} & 
    mortality of patient in ICU & \citet{le1993new} \\
    \bottomrule
    \end{tabular}
    }
    \caption{Clinical prediction tasks considered in \cref{Sec::Experiments} and \cref{Sec::Demonstration}. We state conditions for $y_i = +1$ for each dataset. All datasets used are publicly available. Datasets based on MIMIC-III \cite{johnson2016mimic} (\textds{kidney}, \textds{mortality}) and eICU \cite{pollard2018eicu} (\textds{cardio}) are hosted on PhysioNet under the PhysioNet Credentialed Health Data License. The \textds{heart} dataset is hosted on the UCI ML Repository under an Open Data license. The \textds{apnea} and \textds{saps} datasets must be requested from the authors of the papers listed under references~\citep{le1993new,ustun2016clinical}. In cases where data access requires consent or approval from the data holders, we have followed the proper procedure to obtain such consent.}
    \label{Table::DatasetOverview}
\end{table}

\paragraph{\textds{apnea}} We use the obstructive sleep apnea dataset from~\citet{ustun2016clinical}~\citep[see also][]{ustun2016supersparse}. The dataset contains a cohort of 1,152 patients of which $\prob{y=+1}=23\%$ have OSA and includdes 26 features that cover information that is readily available in an electronic health record (e.g. BMI, comobordities, age, sex).

\paragraph{\textds{cardio\_eicu} \& \textds{cardio\_mimic}} Cardiogenic shock is a serious acute condition where the heart cannot provide sufficient blood to the vital organs. We create a cohort of patients who have cardiogenic shock during an ICU stay from the eICU Collaborative Research Database V2.0\citep{pollard2018eicu} and MIMIC-III databases~\cite{johnson2016mimic}, respectively. The goal is to predict mortality for a patient with cardiogenic shock. As features include summarize statistics for vitals and lab tests (e.g. systolic BP, heart rate, hemoglobin count) obtained up to 24 hours prior to the onset of cardiogenic shock. The final dataset contains 8,815 patients and $\prob{y_i =+1} = 13.5\%$.

\paragraph{\textds{heart}} We use the Heart dataset from the UCI Machine Learning Repository, where the goal is to predict the presence of heart disease which covers a cohort of 303 patients, of which $\prob{y_i = +1} = 54.5\%$ have heart disease. We use all available features, treating \textfn{cp}, \textfn{thal}, \textfn{ca}, \textfn{slope} and \textfn{restecg} as categorical, and all remaining features as continuous. 

\paragraph{\textds{kidney}} We use MIMIC-III and MIMIC-IV \cite{johnson2016mimic} to define a cohort of patients who were given \emph{continuous renal replacement therapy} (CRRT) at any point during their ICU stay. For patients with multiple ICU stays, we select their first one. We define the target as whether the patient dies during the course of their selected hospital admission. As features, we select the most recent instances of relevant lab measurements (e.g. sodium, potassium, creatinine) prior to the CRRT start time, along with the patient's age, the number of hours they have been in ICU when CRRT was administered, and their Sequential Organ Failure Assessment (SOFA) score at admission. We treat all variables as continuous with the exception of the SOFA score, which we treat as ordinal. This results in a dataset of 1,722 CRRT patients, with $\prob{y_i = +1} = 51.1\%.$ %We define protected groups based on the patient's sex and self-reported race and ethnicity.

\paragraph{\textds{mortality}} We define a cohort of patients for in-hospital mortality prediction task following \citet{harutyunyan2019multitask}. We select the first ICU stay longer than 48 hours for patients in MIMIC-III\citep{johnson2016mimic}, and predict in-hospital mortality for this visit. As features, we include periodic lab and vital measurements used by \citet{harutyunyan2019multitask} into four 12-hour time-bins, and compute the mean in each time-bin. This results in a cohort of 21,139 patients where $\prob{y_i = +1} = 13.2\%.$

\paragraph{\textds{saps}} The Simplified Acute Physiology Score II (SAPS II) is a risk score developed to predict ICU mortality~\cite{le1993new}. This study contains a cohort of critically-ill patients from 137 medical centers across 12 countries. For each patient we have access to demographics, comorbidities, and vitals which are used to predict the risk of mortality in the ICU. The final dataset contains 7,797 patients where percentage of patients in the dataset who experience mortality is $\prob{y_i = +1} = 21.8\%$.

\section{Additional Experimental Results}
\label{Appendix::Experiments}

We include additional results showing the gains of personalization when training personalized neural nets and random forests. We present tables that summarize the gains of personalization for neural networks and random forests. The following tables are analogous to \cref{Table::Results}, except that they also include results for the \textds{kidney} dataset in \cref{Sec::Demonstration}. 

\paragraph{Neural Nets}
We trained neural networks with two hidden layers of size 5 and 2 and learning rate of $1^{-3}$. We applied Platt scaling~\cite{platt1999probabilistic} to ensure that the models assigned calibrated probabilities. As in \cref{Sec::ExperimentalResults} and \cref{Sec::Demonstration}, we can identify significant fair use violations and gains as noted by the gains and violations. 

\begin{table}[!htb]
\centering
\resizebox{\linewidth}{!}{
\fontsize{8}{10}\selectfont
\begin{tabular}[t]{@{}l@{}r@{}H@{}H@{}H@{}lllllllll}
\multicolumn{2}{l}{ } & \multicolumn{3}{c}{} & \multicolumn{3}{c}{Test Error} & \multicolumn{3}{c}{Test AUC} & \multicolumn{3}{c}{Test ECE} \\
\cmidrule(l{3pt}r{3pt}){3-5} \cmidrule(l{3pt}r{3pt}){6-8} \cmidrule(l{3pt}r{3pt}){9-11} \cmidrule(l{3pt}r{3pt}){12-14}
\cmidrule(l{-2pt}r{3pt}){6-8} \cmidrule(l{3pt}r{3pt}){9-11} \cmidrule(l{3pt}r{3pt}){12-14}

Dataset & Metrics & 
\onehot{} & \intonehot{} & \dcp{}  & 
\onehot{} & \intonehot{} & \dcp{}  & 
\onehot{} & \intonehot{} & \dcp{} & 
\onehot{} & \intonehot{} & \dcp{}\\

\cmidrule(l{-2pt}r{-2pt}){1-2} \cmidrule(l{-2pt}r{3pt}){6-8} \cmidrule(l{3pt}r{3pt}){9-11} \cmidrule(l{3pt}r{3pt}){12-14}
\apnea{} & \metricsguide{} & \cell{r}{0.631\\-0.012\\0.079 / -0.021\\2/2\\5/4} & \cell{r}{0.689\\-0.070\\0.008 / -0.099\\0/0\\6/6} & \cell{r}{0.671\\-0.053\\-0.024 / -0.188\\0/0\\3/3} & \cell{r}{35.0\%\\-1.7\%\\15.7\% / -4.5\%\\4/3\\3/0} & \cell{r}{48.4\%\\-15.1\%\\-6.1\% / -34.4\%\\6/6\\0/0} & \cell{r}{41.5\%\\-8.1\%\\-2.2\% / -50.5\%\\6/6\\2/1} & \cell{r}{0.704\\-0.012\\0.097 / -0.040\\2/2\\4/5} & \cell{r}{0.502\\-0.215\\-0.052 / -0.496\\0/0\\6/6} & \cell{r}{0.622\\-0.095\\-0.068 / -0.328\\0/0\\4/4} & \cell{r}{4.8\%\\0.8\%\\8.0\% / -9.5\%\\2/2\\1/1} & \cell{r}{2.4\%\\3.2\%\\25.2\% / 3.3\%\\0/0\\0/0} & \cell{r}{5.3\%\\0.4\%\\9.8\% / -5.7\%\\2/2\\1/1}\\
\cmidrule(l{-2pt}r{-2pt}){1-2} \cmidrule(l{-2pt}r{3pt}){6-8} \cmidrule(l{3pt}r{3pt}){9-11} \cmidrule(l{3pt}r{3pt}){12-14}

\cshocke{} & \metricsguide{} & \cell{r}{0.607\\0.006\\0.060 / -0.017\\3/3\\3/3} & \cell{r}{0.608\\0.004\\0.068 / -0.019\\1/1\\4/4} & \cell{r}{0.633\\-0.021\\0.006 / -0.038\\0/0\\1/1} & \cell{r}{31.5\%\\1.6\%\\8.4\% / -0.5\%\\0/0\\2/0} & \cell{r}{31.8\%\\1.3\%\\5.5\% / -1.3\%\\2/1\\3/0} & \cell{r}{36.6\%\\-3.5\%\\0.0\% / -10.3\%\\3/3\\2/2} & \cell{r}{0.739\\0.001\\0.067 / -0.003\\3/3\\4/4} & \cell{r}{0.738\\-0.001\\0.029 / -0.012\\1/1\\3/4} & \cell{r}{0.687\\-0.051\\-0.000 / -0.091\\0/0\\2/2} & \cell{r}{4.5\%\\2.3\%\\2.6\% / -1.2\%\\1/1\\2/2} & \cell{r}{5.5\%\\1.4\%\\2.4\% / -1.9\%\\1/1\\1/1} & \cell{r}{5.4\%\\1.5\%\\5.4\% / -2.8\%\\2/2\\1/1}\\
\cmidrule(l{-2pt}r{-2pt}){1-2} \cmidrule(l{-2pt}r{3pt}){6-8} \cmidrule(l{3pt}r{3pt}){9-11} \cmidrule(l{3pt}r{3pt}){12-14}

\cshockm{} & \metricsguide{} & \cell{r}{0.492\\0.006\\0.014 / -0.005\\3/3\\4/4} & \cell{r}{0.491\\0.007\\0.016 / -0.000\\3/3\\2/2} & \cell{r}{0.498\\-0.000\\0.011 / -0.008\\2/2\\0/0} & \cell{r}{23.7\%\\0.6\%\\2.0\% / -1.1\%\\1/1\\1/1} & \cell{r}{24.0\%\\0.2\%\\2.3\% / -2.4\%\\2/2\\0/0} & \cell{r}{23.9\%\\0.4\%\\1.4\% / -1.3\%\\2/2\\3/3} & \cell{r}{0.849\\0.004\\0.018 / -0.005\\3/3\\3/3} & \cell{r}{0.849\\0.004\\0.012 / -0.000\\3/3\\3/3} & \cell{r}{0.836\\-0.009\\0.003 / -0.015\\1/1\\0/0} & \cell{r}{3.1\%\\1.1\%\\2.1\% / -0.4\%\\1/1\\0/0} & \cell{r}{4.7\%\\-0.4\%\\1.4\% / -2.3\%\\2/2\\0/0} & \cell{r}{3.3\%\\1.0\%\\2.5\% / -0.2\%\\0/0\\1/1}\\
\cmidrule(l{-2pt}r{-2pt}){1-2} \cmidrule(l{-2pt}r{3pt}){6-8} \cmidrule(l{3pt}r{3pt}){9-11} \cmidrule(l{3pt}r{3pt}){12-14}

\heart{} & \metricsguide{} & \cell{r}{0.791\\-0.110\\-0.032 / -0.450\\0/0\\4/4} & \cell{r}{0.644\\0.037\\0.126 / -0.044\\2/2\\3/3} & \cell{r}{0.815\\-0.134\\0.105 / -0.508\\1/1\\4/4} & \cell{r}{50.0\%\\1.3\%\\12.0\% / -12.8\%\\2/1\\2/1} & \cell{r}{26.3\%\\25.0\%\\29.7\% / 16.6\%\\0/0\\2/1} & \cell{r}{38.2\%\\13.2\%\\28.1\% / 7.1\%\\0/0\\3/1} & \cell{r}{0.451\\-0.096\\0.046 / -0.387\\1/1\\1/4} & \cell{r}{0.771\\0.225\\0.393 / 0.119\\4/4\\0/3} & \cell{r}{0.554\\0.007\\0.257 / -0.023\\1/2\\1/2} & \cell{r}{21.3\%\\-7.8\%\\-0.1\% / -27.2\%\\3/3\\1/1} & \cell{r}{19.5\%\\-5.9\%\\16.8\% / -5.7\%\\1/1\\0/0} & \cell{r}{18.1\%\\-4.5\%\\6.2\% / -14.8\%\\1/1\\1/1}\\
\cmidrule(l{-2pt}r{-2pt}){1-2} \cmidrule(l{-2pt}r{3pt}){6-8} \cmidrule(l{3pt}r{3pt}){9-11} \cmidrule(l{3pt}r{3pt}){12-14}

\kidney{} & \metricsguide{} & \cell{r}{0.588\\-0.019\\0.004 / -0.134\\1/1\\6/5} & \cell{r}{0.579\\-0.010\\0.036 / -0.103\\1/1\\3/3} & \cell{r}{0.603\\-0.034\\0.004 / -0.262\\1/1\\2/2} & \cell{r}{29.5\%\\-2.3\%\\1.2\% / -7.8\%\\5/4\\3/0} & \cell{r}{31.7\%\\-4.5\%\\5.2\% / -6.8\%\\5/5\\1/0} & \cell{r}{30.9\%\\-3.7\%\\-1.6\% / -16.3\%\\6/6\\4/4} & \cell{r}{0.758\\-0.013\\0.047 / -0.144\\2/2\\4/6} & \cell{r}{0.774\\0.004\\0.049 / -0.103\\4/4\\3/5} & \cell{r}{0.762\\-0.009\\0.032 / -0.135\\2/2\\2/2} & \cell{r}{5.6\%\\0.3\%\\4.6\% / -7.8\%\\2/2\\1/0} & \cell{r}{6.8\%\\-0.9\%\\1.9\% / -5.6\%\\4/4\\0/0} & \cell{r}{7.3\%\\-1.4\%\\1.0\% / -5.9\%\\5/5\\3/3}\\
\cmidrule(l{-2pt}r{-2pt}){1-2} \cmidrule(l{-2pt}r{3pt}){6-8} \cmidrule(l{3pt}r{3pt}){9-11} \cmidrule(l{3pt}r{3pt}){12-14}

\mortality{} & \metricsguide{} & \cell{r}{0.447\\-0.002\\0.095 / -0.030\\4/4\\4/0} & \cell{r}{0.457\\-0.012\\0.000 / -0.042\\0/0\\2/2} & \cell{r}{0.409\\0.036\\0.226 / 0.010\\6/6\\0/0} & \cell{r}{20.4\%\\0.1\%\\5.2\% / -1.7\%\\2/2\\5/1} & \cell{r}{21.6\%\\-1.1\%\\-0.6\% / -3.2\%\\6/6\\2/2} & \cell{r}{17.7\%\\2.8\%\\12.9\% / 0.0\%\\0/0\\6/6} & \cell{r}{0.870\\-0.003\\0.032 / -0.018\\3/3\\0/4} & \cell{r}{0.869\\-0.004\\-0.000 / -0.022\\0/0\\4/4} & \cell{r}{0.895\\0.022\\0.042 / 0.005\\6/6\\0/0} & \cell{r}{2.8\%\\0.6\%\\2.7\% / -0.8\%\\3/3\\4/1} & \cell{r}{4.7\%\\-1.3\%\\2.9\% / -1.8\%\\3/3\\3/3} & \cell{r}{3.0\%\\0.5\%\\8.3\% / 0.1\%\\0/0\\6/6}\\
\cmidrule(l{-2pt}r{-2pt}){1-2} \cmidrule(l{-2pt}r{3pt}){6-8} \cmidrule(l{3pt}r{3pt}){9-11} \cmidrule(l{3pt}r{3pt}){12-14}

\saps{} & \metricsguide{} & \cell{r}{1.185\\0.114\\0.274 / -0.018\\3/3\\3/3} & \cell{r}{0.477\\0.822\\0.956 / 0.285\\4/4\\2/2} & \cell{r}{0.967\\0.332\\0.693 / -3.656\\3/3\\3/3} & \cell{r}{53.9\%\\7.7\%\\13.1\% / 0.0\%\\2/0\\4/1} & \cell{r}{22.5\%\\39.0\%\\54.8\% / 1.4\%\\1/0\\3/0} & \cell{r}{48.9\%\\12.7\%\\22.0\% / 0.0\%\\1/0\\3/2} & \cell{r}{0.521\\0.328\\0.727 / 0.197\\4/4\\1/3} & \cell{r}{0.872\\0.679\\0.757 / 0.638\\4/4\\1/3} & \cell{r}{0.758\\0.565\\0.743 / -0.273\\3/3\\3/4} & \cell{r}{43.6\%\\1.7\%\\13.2\% / 1.6\%\\0/0\\0/0} & \cell{r}{9.4\%\\36.0\%\\45.1\% / -2.9\%\\1/1\\1/1} & \cell{r}{31.5\%\\13.9\%\\49.9\% / 6.4\%\\0/0\\1/1}\\
\cmidrule(l{-2pt}r{-2pt}){1-2} \cmidrule(l{-2pt}r{3pt}){6-8} \cmidrule(l{3pt}r{3pt}){9-11} \cmidrule(l{3pt}r{3pt}){12-14}
\end{tabular}}
\caption{Gains of personalization for neural network models on test data.}
\end{table}

\FloatBarrier
\paragraph{Random Forests}
We trained random forests with the following hyperparameters: 100 estimators, max depth of 20, minimum samples per split is 5, and minimum number of samples in each leaf is 2. We expect these models to perform well in terms of error rare but not necessarily in terms in terms of AUC or risk calibration. We observe this effect in the Table below. For example, using an intersectional encoding with random forests minimizing fair use violations in terms of error rate as measured on multiple datasets (e.g. \textds{apnea}, \textds{kidney}). As noted with other model classes, we can find  statistically significant violations.

\begin{table}[!htb]
\centering
\resizebox{\linewidth}{!}{
\fontsize{8}{10}\selectfont
\begin{tabular}{@{}l@{}r@{}H@{}H@{}H@{}lllllllll}
\multicolumn{2}{l}{ } & \multicolumn{3}{c}{} & \multicolumn{3}{c}{Test Error} & \multicolumn{3}{c}{Test AUC} & \multicolumn{3}{c}{Test ECE} \\
\cmidrule(l{3pt}r{3pt}){3-5} \cmidrule(l{3pt}r{3pt}){6-8} \cmidrule(l{3pt}r{3pt}){9-11} \cmidrule(l{3pt}r{3pt}){12-14}
\cmidrule(l{-2pt}r{3pt}){6-8} \cmidrule(l{3pt}r{3pt}){9-11} \cmidrule(l{3pt}r{3pt}){12-14}

Dataset & Metrics & 
\onehot{} & \intonehot{} & \dcp{}  & 
\onehot{} & \intonehot{} & \dcp{}  & 
\onehot{} & \intonehot{} & \dcp{} & 
\onehot{} & \intonehot{} & \dcp{}\\

\cmidrule(l{-2pt}r{-2pt}){1-2} \cmidrule(l{-2pt}r{3pt}){6-8} \cmidrule(l{3pt}r{3pt}){9-11} \cmidrule(l{3pt}r{3pt}){12-14}

\apnea() & \metricsguide{} & \cell{r}{0.602\\-0.004\\0.010 / -0.012\\1/1\\6/3} & \cell{r}{0.602\\-0.005\\0.000 / -0.009\\0/0\\1/1} & \cell{r}{0.550\\0.048\\0.112 / -0.010\\5/5\\0/0} & \cell{r}{29.7\%\\1.8\%\\4.7\% / -4.4\%\\2/2\\3/0} & \cell{r}{31.0\%\\0.3\%\\1.4\% / -3.8\%\\4/1\\3/1} & \cell{r}{26.5\%\\5.3\%\\17.0\% / -6.0\%\\1/1\\5/5} & \cell{r}{0.751\\-0.004\\0.061 / -0.021\\1/2\\2/6} & \cell{r}{0.757\\-0.001\\0.019 / -0.015\\1/2\\2/4} & \cell{r}{0.815\\0.055\\0.104 / -0.008\\5/5\\1/1} & \cell{r}{8.2\%\\-2.3\%\\2.2\% / -3.9\%\\3/3\\0/0} & \cell{r}{7.2\%\\-1.1\%\\2.1\% / -2.0\%\\3/3\\1/1} & \cell{r}{8.0\%\\-1.0\%\\2.8\% / -2.6\%\\2/2\\0/0}\\
\cmidrule(l{-2pt}r{-2pt}){1-2} \cmidrule(l{-2pt}r{3pt}){6-8} \cmidrule(l{3pt}r{3pt}){9-11} \cmidrule(l{3pt}r{3pt}){12-14}

\cshocke{} & \metricsguide{} & \cell{r}{0.597\\0.000\\0.006 / -0.005\\2/2\\2/2} & \cell{r}{0.598\\-0.001\\0.005 / -0.006\\2/2\\2/2} & \cell{r}{0.564\\0.032\\0.166 / 0.004\\4/4\\0/0} & \cell{r}{30.8\%\\0.4\%\\3.6\% / -3.6\%\\2/2\\1/0} & \cell{r}{30.5\%\\-0.3\%\\0.0\% / -0.9\%\\4/1\\2/2} & \cell{r}{27.1\%\\3.9\%\\16.4\% / 0.4\%\\0/0\\4/4} & \cell{r}{0.770\\0.003\\0.016 / -0.008\\2/2\\2/3} & \cell{r}{0.769\\0.003\\0.013 / -0.012\\2/2\\1/1} & \cell{r}{0.801\\0.032\\0.121 / 0.007\\4/4\\0/0} & \cell{r}{8.0\%\\-0.5\%\\0.7\% / -2.0\%\\3/3\\1/1} & \cell{r}{8.5\%\\-0.3\%\\0.5\% / -0.0\%\\0/0\\1/1} & \cell{r}{9.4\%\\-1.9\%\\4.2\% / -1.4\%\\2/2\\0/0}\\
\cmidrule(l{-2pt}r{-2pt}){1-2} \cmidrule(l{-2pt}r{3pt}){6-8} \cmidrule(l{3pt}r{3pt}){9-11} \cmidrule(l{3pt}r{3pt}){12-14}

\cshockm{} & \metricsguide{} & \cell{r}{0.523\\-0.003\\-0.002 / -0.004\\0/0\\1/1} & \cell{r}{0.524\\-0.006\\-0.002 / -0.008\\0/0\\2/2} & \cell{r}{0.500\\0.019\\0.046 / 0.009\\4/4\\0/0} & \cell{r}{24.0\%\\-0.3\%\\0.9\% / -1.3\%\\2/2\\2/0} & \cell{r}{23.7\%\\0.3\%\\0.6\% / -0.1\%\\1/0\\3/3} & \cell{r}{20.9\%\\2.9\%\\5.8\% / 1.1\%\\0/0\\4/4} & \cell{r}{0.849\\0.001\\0.003 / -0.002\\3/3\\2/2} & \cell{r}{0.850\\-0.002\\0.004 / -0.004\\1/1\\1/1} & \cell{r}{0.871\\0.023\\0.047 / 0.007\\4/4\\0/0} & \cell{r}{10.0\%\\-0.8\%\\0.5\% / -1.7\%\\3/3\\0/0} & \cell{r}{11.0\%\\-0.9\%\\-0.0\% / -1.6\%\\3/3\\1/1} & \cell{r}{11.7\%\\-2.3\%\\-0.1\% / -4.6\%\\3/3\\0/0}\\
\cmidrule(l{-2pt}r{-2pt}){1-2} \cmidrule(l{-2pt}r{3pt}){6-8} \cmidrule(l{3pt}r{3pt}){9-11} \cmidrule(l{3pt}r{3pt}){12-14}

\heart{} & \metricsguide{} & \cell{r}{0.425\\-0.004\\-0.001 / -0.011\\0/0\\1/1} & \cell{r}{0.427\\0.001\\0.012 / -0.020\\2/2\\1/1} & \cell{r}{0.372\\0.056\\0.137 / -0.061\\3/3\\0/0} & \cell{r}{18.4\%\\1.3\%\\5.9\% / 0.0\%\\3/0\\4/0} & \cell{r}{21.1\%\\2.6\%\\10.8\% / 0.0\%\\3/0\\3/1} & \cell{r}{21.1\%\\-1.3\%\\16.3\% / -18.6\%\\3/2\\1/1} & \cell{r}{0.899\\0.001\\0.004 / -0.067\\0/3\\0/4} & \cell{r}{0.896\\-0.000\\0.016 / -0.063\\1/3\\1/4} & \cell{r}{0.936\\0.035\\0.094 / 0.001\\3/4\\0/0} & \cell{r}{9.2\%\\2.0\%\\7.1\% / -3.7\%\\2/2\\2/2} & \cell{r}{10.6\%\\4.5\%\\4.3\% / -4.6\%\\2/2\\4/4} & \cell{r}{13.5\%\\1.1\%\\11.3\% / -13.8\%\\2/2\\2/2}\\
\cmidrule(l{-2pt}r{-2pt}){1-2} \cmidrule(l{-2pt}r{3pt}){6-8} \cmidrule(l{3pt}r{3pt}){9-11} \cmidrule(l{3pt}r{3pt}){12-14}

\kidney{} & \metricsguide{} & \cell{r}{0.593\\0.001\\0.004 / -0.011\\4/4\\4/4} & \cell{r}{0.595\\-0.002\\0.005 / -0.005\\1/1\\2/2} & \cell{r}{0.529\\0.065\\0.125 / 0.033\\6/6\\0/0} & \cell{r}{30.1\%\\-0.4\%\\0.5\% / -3.4\%\\4/2\\5/0} & \cell{r}{30.5\%\\-1.2\%\\0.0\% / -3.4\%\\6/4\\1/0} & \cell{r}{22.3\%\\7.8\%\\17.8\% / 3.1\%\\0/0\\6/6} & \cell{r}{0.773\\-0.003\\0.008 / -0.022\\2/2\\3/5} & \cell{r}{0.773\\-0.005\\0.015 / -0.008\\2/2\\3/3} & \cell{r}{0.860\\0.083\\0.143 / 0.062\\6/6\\0/0} & \cell{r}{7.5\%\\0.8\%\\1.8\% / -1.8\%\\3/3\\2/0} & \cell{r}{7.5\%\\1.3\%\\1.7\% / -2.6\%\\1/1\\0/0} & \cell{r}{13.2\%\\-5.1\%\\-1.7\% / -8.1\%\\6/6\\0/0}\\
\cmidrule(l{-2pt}r{-2pt}){1-2} \cmidrule(l{-2pt}r{3pt}){6-8} \cmidrule(l{3pt}r{3pt}){9-11} \cmidrule(l{3pt}r{3pt}){12-14}

\addlinespace
\mortality{} & \metricsguide{} & \cell{r}{0.578\\-0.004\\-0.003 / -0.008\\0/0\\5/1} & \cell{r}{0.576\\-0.000\\0.004 / -0.007\\3/3\\6/0} & \cell{r}{0.548\\0.029\\0.299 / 0.009\\6/6\\0/0} & \cell{r}{27.1\%\\0.2\%\\0.8\% / -1.1\%\\4/4\\5/0} & \cell{r}{26.9\%\\0.4\%\\1.0\% / -0.7\%\\2/2\\6/0} & \cell{r}{24.6\%\\2.4\%\\22.2\% / 0.2\%\\0/0\\6/6} & \cell{r}{0.803\\-0.004\\0.004 / -0.011\\1/1\\0/6} & \cell{r}{0.806\\0.002\\0.009 / -0.011\\3/3\\0/6} & \cell{r}{0.841\\0.035\\0.186 / 0.013\\6/6\\0/0} & \cell{r}{11.0\%\\-0.5\%\\0.2\% / -1.2\%\\3/3\\3/0} & \cell{r}{10.9\%\\-0.6\%\\0.3\% / -1.3\%\\4/4\\6/0} & \cell{r}{12.0\%\\-0.9\%\\0.7\% / -7.8\%\\4/4\\0/0}\\
\cmidrule(l{-2pt}r{-2pt}){1-2} \cmidrule(l{-2pt}r{3pt}){6-8} \cmidrule(l{3pt}r{3pt}){9-11} \cmidrule(l{3pt}r{3pt}){12-14}

\saps{} & \metricsguide{} & \cell{r}{0.447\\-0.002\\0.012 / -0.006\\2/2\\2/1} & \cell{r}{0.448\\-0.003\\0.010 / -0.003\\3/3\\3/3} & \cell{r}{0.441\\0.006\\0.059 / -0.068\\1/1\\2/2} & \cell{r}{19.6\%\\0.0\%\\0.1\% / -0.2\%\\3/1\\2/1} & \cell{r}{19.8\%\\0.0\%\\22.2\% / -0.4\%\\1/1\\3/1} & \cell{r}{19.1\%\\0.6\%\\4.6\% / 0.0\%\\1/0\\3/3} & \cell{r}{0.880\\-0.001\\0.000 / -0.001\\0/2\\0/2} & \cell{r}{0.880\\0.000\\0.002 / -0.023\\1/2\\1/3} & \cell{r}{0.882\\0.002\\0.023 / -0.187\\2/2\\1/2} & \cell{r}{5.0\%\\-0.5\%\\1.9\% / -0.6\%\\2/2\\1/0} & \cell{r}{4.9\%\\-0.5\%\\0.3\% / -1.0\%\\2/2\\1/1} & \cell{r}{4.7\%\\0.3\%\\8.7\% / -0.4\%\\1/1\\3/3}\\
\cmidrule(l{-2pt}r{-2pt}){1-2} \cmidrule(l{-2pt}r{3pt}){6-8} \cmidrule(l{3pt}r{3pt}){9-11} \cmidrule(l{3pt}r{3pt}){12-14}
\end{tabular}}
\caption{Performance of personalized random forests models on all datasets. We describe the metrics shown for each model and dataset in \cref{Table::Results}.}
\end{table}

\end{document}